\documentclass{article}

\usepackage{microtype}
\usepackage{graphicx}
\usepackage{subfigure}
\usepackage{booktabs} 

\usepackage{hyperref}

\usepackage[ruled, lined, longend, algo2e]{algorithm2e}

\usepackage[accepted]{icml2024}

\usepackage{amsmath}
\usepackage{amssymb}
\usepackage{mathtools}
\usepackage{amsthm}
\usepackage{bbm}
\PassOptionsToPackage{normalem}{ulem}
\usepackage{ulem}

\usepackage[capitalize,noabbrev]{cleveref}

\usepackage{enumitem}
\setlist{leftmargin=2mm,itemsep=0mm,topsep=0mm,parsep=0mm,partopsep=0mm}

\theoremstyle{plain}
\newtheorem{theorem}{Theorem}[section]

\newtheorem{lemma}[theorem]{Lemma}

\newtheorem{condition}[theorem]{Condition}
\theoremstyle{definition}

\theoremstyle{remark}

\usepackage[textsize=tiny]{todonotes}
\newcommand{\N}{\texttt{N}}
\newcommand{\Y}{\texttt{Y}}
\newcommand{\Z}{\texttt{Z}}
\newcommand{\W}{\texttt{W}}
\newcommand{\binned}{\texttt{bin}}
\newcommand{\num}{\texttt{num}}
\newcommand{\denom}{\texttt{den}}

\newcommand{\Tr}{\texttt{Tr}}
\newcommand{\Tt}{\texttt{Ev}}

\newcommand{\E}{\mathbb{E}}

\newcommand{\h}[1]{\hat{#1}}

\newcommand{\fracdzero}{\frac{\mathbbm{1}\{d=0\}}{p(d=0)}}
\newcommand{\fracdone}{\frac{\mathbbm{1}\{d=1\}}{p(d=1)}}

\SetCommentSty{mycommfont}

\icmltitlerunning{Explaining performance gaps}

\begin{document}

\twocolumn[
\icmltitle{A hierarchical decomposition for explaining ML performance discrepancies}

\icmlsetsymbol{equal}{*}

\begin{icmlauthorlist}
\icmlauthor{Jean Feng}{yyy}
\icmlauthor{Harvineet Singh}{yyy}
\icmlauthor{Fan Xia}{yyy}
\icmlauthor{Adarsh Subbaswamy}{comp}
\icmlauthor{Alexej Gossmann}{comp}
\end{icmlauthorlist}

\icmlaffiliation{yyy}{University of California, San Francisco, USA}
\icmlaffiliation{comp}{U.S. Food and Drug Administration, Center for Devices and Radiological Health}

\icmlcorrespondingauthor{Jean Feng}{jean.feng@ucsf.edu}

\icmlkeywords{Distribution shift, Feature attribution, Explainability, Model generalizability, Nonparametric inference, Debiased Machine Learning}

\vskip 0.3in
]

\printAffiliationsAndNotice{}

\begin{abstract}
Machine learning (ML) algorithms can often differ in performance across domains.
Understanding \textit{why} their performance differs is crucial for determining what types of interventions (e.g., algorithmic or operational) are most effective at closing the performance gaps.
Existing methods focus on \textit{aggregate decompositions} of the total performance gap into the impact of a shift in the distribution of features $p(X)$ versus the impact of a shift in the conditional distribution of the outcome $p(Y|X)$; however, such coarse explanations offer only a few options for how one can close the performance gap.
\textit{Detailed variable-level decompositions} that quantify the importance of each variable to each term in the aggregate decomposition can provide a much deeper understanding and suggest much more targeted interventions.
However, existing methods assume knowledge of the full causal graph or make strong parametric assumptions.
We introduce a nonparametric hierarchical framework that provides both aggregate and detailed decompositions for explaining why the performance of an ML algorithm differs across domains, without requiring causal knowledge.
We derive debiased, computationally-efficient estimators, and statistical inference procedures for asymptotically valid confidence intervals.
\end{abstract}

\vspace{-0.5cm}
\section{Introduction}
\label{sec:intro}

The performance of an ML algorithm can differ across domains due to shifts in the data distribution.
To understand what contributed to this performance gap, prior works have suggested decomposing the gap into that due to a shift in the marginal distribution of the input features $p(X)$ versus that due to a shift in the conditional distribution of the outcome $p(Y|X)$ \citep{Cai2023-ov, zhang2023why, Liu2023-gs, Qiu2023-bn, firpo2018decomposing}.
Although such \textit{aggregate} decompositions can be helpful, more \textit{detailed} explanations that quantify the importance of each variable to each term in this decomposition can provide deeper insight to ML teams trying to understand and close the performance gap.
However, there is currently no principled framework that provides both \textit{aggregate} and \textit{detailed} decompositions for explaining performance gaps of ML algorithms.

As a motivating example, suppose an ML deployment team has an algorithm that predicts the risk of patients being readmitted to the hospital given data from the Electronic Health Records (EHR), e.g. demographic variables and diagnosis codes.
The algorithm was trained for a general patient population.
The team intends to deploy it for heart failure (HF) patients and observes a large performance drop (e.g. in accuracy) in this subgroup.
An \textit{aggregate} decomposition of the performance drop into the marginal versus conditional components ($p(X)$ and $p(Y|X)$) provides only a high-level understanding of the underlying cause and limited suggestions on how the model can be fixed.
A small conditional term in the aggregate decomposition but a large marginal term is typically addressed by retraining the model on a reweighted version of the original data that matches the target population \citep{quionero2009dataset}; otherwise, the standard suggestion is to collect more data from the target population to recalibrate/fine-tune the algorithm \citep{Steyerberg2009-ze}.
A \textit{detailed} decomposition would help the deployment team conduct a root cause analysis and assess the utility of targeted variable-specific interventions.
For instance, if the performance gap is primarily driven by a marginal shift in a few diagnoses, the team can investigate why the rates of these diagnoses differ between general and HF patients.
The team may find that certain diagnoses differ due to variations in patient case mix, which may be better addressed through model retraining, whereas others are due to variations in documentation practices, which may be better addressed through operational interventions.

Existing approaches for providing a detailed understanding of why the performance of an ML algorithm differs across domains fall into two general categories.
One category, which is also the most common in the applied literature, is to quantify how much the data distribution shifted \citep{Liu2023-gs,budhathoki2021distribution,kulinski2020detection,rabanser2019failing}, e.g. one can compare how the distribution of each input variable differs across domains \citep{cummings2023external}.
However, given the complex interactions and non-linearities in (black-box) ML algorithms, it is difficult to quantify how exactly such shifts contribute to changes in performance.
For instance, a large shift with respect to a given variable does not necessarily translate to a large shift in model performance, as that variable may have low importance in the algorithm or the performance metric may be insensitive to that shift.
Thus the other category of approaches---which is also the focus of this manuscript---is to directly quantify how a shift with respect to a variable (subset) influences model performance.
Existing proposals only provide detailed variable-level breakdowns for either the marginal or conditional terms in the aggregate decomposition, but \textit{no unified framework exists}.
Moreover, existing methods either require knowledge of the causal graph relating all variables \citep{zhang2023why} or strong parametric assumptions \citep{Wu2021-dl,Dodd2003-fb}.

\begin{figure}
    \centering
    \vspace{-0.2cm}
    \includegraphics[width=0.32\textwidth]{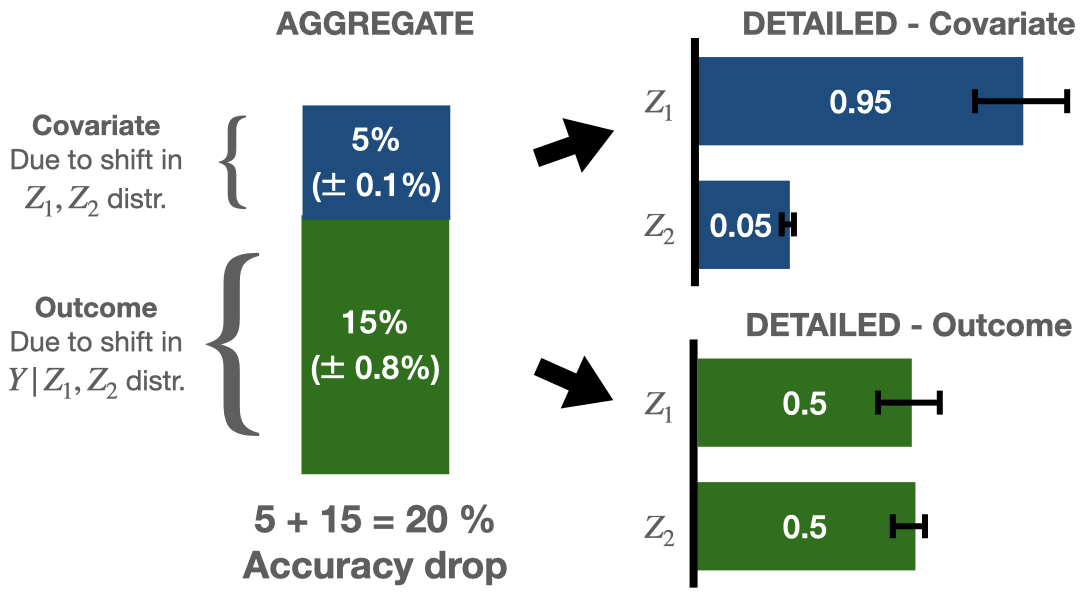}
    \vspace{-0.4cm}
    \caption{\uline{H}ierarchical \uline{D}ecomposition of \uline{P}erformance \uline{D}ifferences (HDPD): \textbf{Aggregate} decomposes a drop in model performance between two domains into that due to shifts in the covariate versus outcome distribution. \textbf{Detailed} quantifies the proportion of variation in accuracy changes explained by partial covariate/outcome shifts with respect to a variable or variable subset.}
    \vspace{-0.7cm}
    \label{fig:overview}
\end{figure}

This work introduces a unified nonparametric framework for explaining the performance gap of an ML algorithm that first decomposes the gap into terms at the aggregate level and then provides detailed variable importance (VI) values for each aggregate term (Fig~\ref{fig:overview}).
Whereas prior works only provide point estimates for the decomposition, we derive debiased, asymptotically linear estimators for the terms in the decomposition which allow for the construction of confidence intervals (CIs) with asymptotically valid coverage rates.
Uncertainty quantification is crucial in this setting, as one often has limited labeled data from the target domain.
The estimation and statistical inference procedure is computationally efficient, despite the exponential number of terms in the Shapley value definition.
We demonstrate the utility of our framework in real-world examples of prediction models for hospital readmission and insurance coverage.

\vspace{-0.4cm}
\section{Prior work}

\textbf{Describing distribution shifts.} This line of work focuses on detecting and localizing which distributions shift between datasets \citep{kulinski2020detection,rabanser2019failing}. \citet{budhathoki2021distribution} identify the main variables contributing  to a distribution shift via a Shapley framework, \citet{kulinski2023explaining} fits interpretable optimal transport maps, and \citet{Liu2023-gs} finds the region with the largest shift in the conditional outcome distribution. However, these works do not quantify how these shifts contribute to \textit{changes in performance}, the metric of practical importance.

\textbf{Explaining loss differences across subpopulations.} Understanding differences in model performance across subpopulations in a single dataset is similar to understanding differences in model performance across datasets, but the focus is typically to find subpopulations with poor performance rather than to explain how distribution shifts contributed to the performance change.
Existing approaches include slice discovery methods \citep{plumb2023towards,jain2023distilling,deon2021spotlight,eyuboglu2022domino} and structured representations of the subpopulation using e.g. Euclidean balls \citep{ali2022lifecycle}. 

\textbf{Attributing performance changes.} Prior works have described similar aggregate decompositions of the performance change into covariate and conditional outcome shift components \citep{Cai2023-ov,Qiu2023-bn}. 
To provide more granular explanations of performance shifts, existing works quantify the importance of shifts in each variable assuming the true causal graph is known \citep{zhang2023why}; covariate shifts restricted to variable subsets assuming partial shifts follow a particular structure \citep{Wu2021-dl}; and conditional shifts in each variable assuming a parametric model \citep{Dodd2003-fb}.
However, the strong assumptions made by these methods make them difficult to apply in practice, and model misspecification can lead to unintuitive interpretations.
In addition, there is no unifying framework for decomposing both covariate and outcome shifts, and many methods do not output CIs, which is important when the amount of labeled data from a given domain is limited.
A summary of how the proposed framework compares against prior works is shown in Table~\ref{tab:comparison} of the Appendix.

\textbf{Decomposition methods in econometrics.}
Explaining performance differences is similar to the long-studied problem of explaining income and health disparities between groups in econometrics.
There, researchers regularly use frameworks such as the Oaxaca-Blinder decomposition \citep{oaxaca1973differentials,blinder1973wage,fortin2011decomposition}, which decomposes disparities into components quantifying the impact of covariate shifts and outcome shifts.
These frameworks commonly decompose the components further to describe the importance of each variable \citep{oaxaca1973differentials,kirby2006health}.
Existing methods typically rely on strong parametric assumptions \citep{fairlie2005extension,yun2004moment,firpo2018decomposing}, which is inappropriate for the complex data settings in ML.

In summary, the distinguishing contribution of this work is that it \textit{unifies} aggregate and detailed decompositions under a \textit{nonparametric} framework with \textit{uncertainty quantification}.

\vspace{-0.3cm}
\section{A hierarchical explanation framework}
Here we introduce how the aggregate and detailed decompositions are defined for explaining the performance gap of a risk prediction algorithm $f: \mathcal{X} \subseteq \mathbb{R}^{m} \to \mathcal{Y}$ for binary outcomes across source and target domains, denoted by $D = 0$ and $D=1$, respectively.
Let the performance of $f$ be quantified in terms of a loss function $\ell: \mathcal{X} \times \mathcal{Y} \to \mathbb{R}$, such as the 0-1 misclassification loss $\mathbbm{1}\{f(X) \ne Y\}$.
Suppose variables $X$ can be partitioned into $W \in \mathbb{R}^{m_1}$ and $Z = X\setminus W  \in \mathbb{R}^{m_2}$, where $m = m_1 + m_2$.
This partitioning allows for a factorization of the data distribution $p_D$ in the source domain $D=0$ and target domain $D=1$ into
\vspace{-0.14cm}
\begin{align}
    p_D(W)p_D(Z|W)p_D(Y|W,Z)
    \label{eq:factor}
\end{align}
(see Fig~\ref{fig:dag} top left).
As such, we refer to $W$ as baseline variables and $Z$ as conditional covariates.
For the readmission example, $W$ may refer to demographics and $Z$ to diagnosis codes.
The total change in performance of $f$ can thus be written as
\smash{$
    \Lambda =
    \mathbb{E}_{111}\left[\ell(W, Z, Y) \right] - \mathbb{E}_{000}\left[\ell(W, Z, Y) \right],
$}
where  $\mathbb{E}_{D_{\W} D_{\Z} D_{\Y}}$ denotes the expectation over the distribution \smash{$p_{D_{\W}}(W) p_{D_{\Z}}(Z|W) p_{D_{\Y}}(Y|W,Z)$}.
A summary of notation used is provided in Table~\ref{tab:notation} of the Appendix.

\textbf{Aggregate.} At the aggregate level, the framework quantifies how replacing each factor in \eqref{eq:factor} from source to target contributes to the performance gap.
We refer to such replacements as ``aggregate shifts,'' as the shift is not restricted to a particular subset of variables.
This leads to the aggregate decomposition $\Lambda=\Lambda_\W + \Lambda_\Z + \Lambda_\Y$, where $\Lambda_\W$ quantifies the impact of a shift in the baseline distribution $p(W)$ (also known as marginal or covariate shifts \citep{quionero2009dataset}), $\Lambda_\Z$ quantifies the impact of a shift in the conditional covariate distribution $p(Z|W)$, and $\Lambda_{\Y}$ quantifies the impact of a shift in the outcome distribution $p(Y|W,Z)$ (also known as concept shift).
More concretely,
\begin{align*}
    \Lambda_\W &= \mathbb{E}_{100}\left[\ell \right] - \mathbb{E}_{000}\left[\ell \right]\\
    \Lambda_\Z &= \mathbb{E}_{1\cdot\cdot}\big[\underbrace{\mathbb{E}_{\cdot10}\left[\ell \mid W\right] - \mathbb{E}_{\cdot 0 0}\left[\ell\mid W\right]}_{\Delta_{\cdot 10}(W)} \big]\\
    \Lambda_\Y &= \mathbb{E}_{11\cdot}\big[ \underbrace{\mathbb{E}_{\cdot\cdot 1}\left[\ell \mid W,Z \right] - \mathbb{E}_{\cdot\cdot0}\left[\ell\mid W,Z\right]}_{\Delta_{\cdot\cdot 0}(W,Z)} \big].
\end{align*}
Prior works have proposed similar aggregate decompositions \citep{Cai2023-ov,Liu2023-gs,firpo2018decomposing}.

\textbf{Detailed.} At the detailed level, the framework outputs Shapley-based variable attributions that describe how shifts with respect to each variable contribute to each term in the aggregate decomposition.
Based on the VI values, an ML development team can better understand the underlying cause for a performance gap and brainstorm which mixture of operational interventions (e.g. changing data collection) and algorithmic interventions (e.g. retraining the model with respect to the variable(s)) would be most effective at closing the performance gap.
For instance, a variable with high importance to the conditional covariate shift term $\Lambda_{\Z}$ can be due to differences in missingness rates, prevalence, or selection bias; and a variable with high importance to the conditional outcome shift term $\Lambda_{\Y}$ may indicate inherent differences in the conditional distribution (also known as effect modification), differences in measurement error or outcome definitions, or omission of variables predictive of the outcome.
Note that the framework does not output a detailed decomposition of the baseline shift term $\Lambda_{\W}$, for reasons we discuss later.

We leverage the Shapley attribution framework for its axiomatic properties, which result in VI values with intuitive interpretations  \citep{Shapley1953-ld,Charnes1988-mi}.
Recall that for a real-valued value function $v$ defined for all subsets $s \subseteq \{1,\cdots,m\}$, the attribution to element $j$ is defined as the average gain in value due to the inclusion of $j$ to every possible subset, i.e.
\begin{align*}
    \phi_{j} :=& \ \frac{1}{m}\sum_{s \subseteq \{1,\cdots,m\}\setminus j } \binom{m-1}{\lvert s \rvert}^{-1}\{v(s \cup j) - v(s)\}.
\end{align*}
So to define a detailed decomposition, the key question is how to define the value of a \textit{partial} distribution shift only with respect to a variable subset $s$; henceforth we refer to such shifts as \textit{$s$-partial shifts}.
It turns out that the answer is far from straightforward.

\begin{figure}
    \centering
    \vspace{-0.2cm}
    \includegraphics[width=0.4\textwidth]{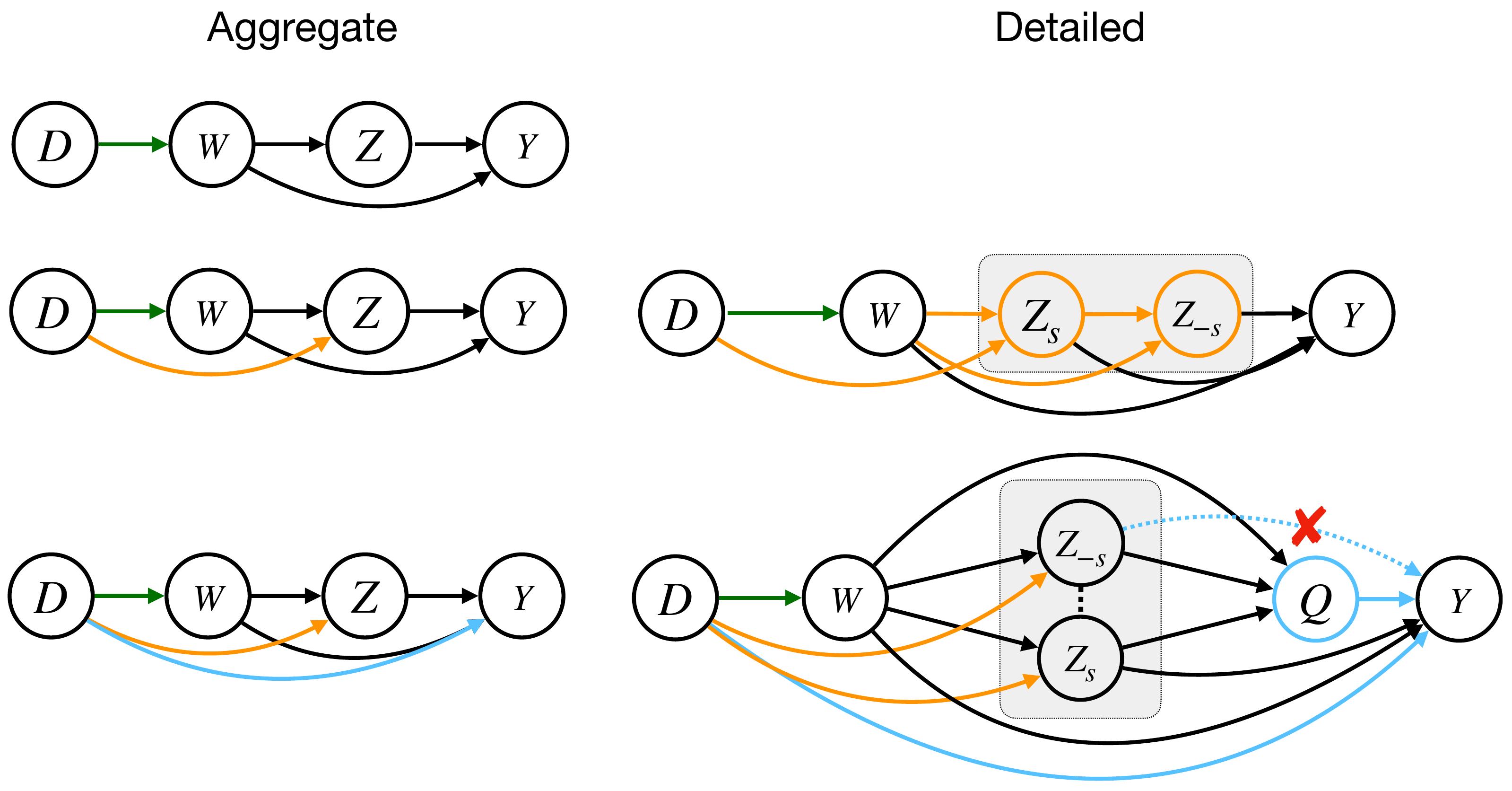}
    \vspace{-0.4cm}
    \caption{
    Decomposition framework for explaining the performance gap from source domain (\smash{$D = 0$}) to target domain (\smash{$D=1$}), visualized through directed acyclic graphs.
    Aggregate decompositions describe the incremental impact of replacing each aggregate variable's distribution at the source with that in the target, indicated by arrows from $D$.
    Detailed decompositions quantify how well hypothesized partial distribution shifts with respect to variable subsets $Z_s$ explain performance gaps.
    This work considers partial outcome shifts that fine-tune the risk in the source domain with respect to $Z_{s}$ (as indicated by the additional node \smash{$Q=p_0(Y=1|W,Z)$}) and partial conditional covariate shifts when \smash{$Z_s \rightarrow Z_{-s}$}.
    }
    \vspace{-0.4cm}
    \label{fig:dag}
\end{figure}

\subsection{Value of partial distribution shifts}

If the true causal graph is known, it would be straightforward to determine how the data distribution shifted with respect to a given variable subset $s$: we would assign the value of subset $s$ as the change in overall performance due to shifts only in $s$ with respect to this graph.
However, in the absence of more detailed causal knowledge, one can only hypothesize the form of partial distribution shifts.
Some proposals define importance as the change in average performance due to hypothesized partial shifts (see e.g. \citet{Wu2021-dl}), but it can lead to unintuitive behavior where a hypothesized distribution shift inconsistent with the true causal graph is attributed high importance.
We illustrate such an example in a simulation in Section~\ref{sec:simulations}.

Instead, our proposal defines the importance of variable subset $s$ as the proportion of \textit{variation} in performance changes across strata explained by its corresponding hypothesized $s$-partial shift, denoted by $p_{s}$. The definition generalizes the traditional $R^2$ measure in statistics, which quantifies how well do covariates explain variation in outcome rather than in performance changes. Prior works on variable importance have leveraged similar definitions \citep{williamson2021nonparametric,hines2023variable}.
For the conditional covariate decomposition, we define the value of $s$ as
\begin{align}
    v_{\Z}(s) &:= 1 - \frac{
    \mathbb{E}_{1\cdot \cdot}\left[\left(\Delta_{\cdot s 0}(W) - \Delta_{\cdot 10}(W) \right)^2\right]
    }{
        \mathbb{E}_{1\cdot\cdot}\left[
            \Delta_{\cdot 10}^2(W)
        \right]
    }.
\end{align}
where
$\tiny
\Delta_{\cdot D_{\Z} 0}(W) = 
\mathbb{E}_{\cdot D_{\Z} 0}\left[\ell | W\right]
    - \mathbb{E}_{\cdot 0 0}\left[\ell | W\right]$
is the performance gap observed at in strata $W$ when $p_0(Z|W)$ is replaced with $p_{D_\Z}(Z|W)$.
(Note that we overload the notation where $D_{\Z}=0$ means source domain, $D_{\Z}=1$ means target domain, and $D_{\Z}=s$ means an $s$-partial shift.)
Setting $\phi_{\Z,0} = v_{\Z}(\emptyset)$, we denote the corresponding Shapley values by $\{\phi_{\Z,j}: j = 0,\cdots, m_2\}$.
Similarly, for the conditional outcome decomposition, we define the importance of $s$ as how well the hypothesized $s$-partial outcome shift explains the variation in performance gaps across strata $(W,Z)$, i.e.
\begin{align*}
\footnotesize
    v_{\Y}(s) &:= 1 - \frac{
    \mathbb{E}_{11\cdot}\left[\left(\Delta_{\cdot \cdot s}(W,Z) - \Delta_{\cdot\cdot 1}(W,Z) \right)^2 \right]
    }{
        \mathbb{E}_{11\cdot}\left[
        \Delta_{\cdot\cdot1}^2(W,Z)
        \right]
    }
\end{align*}
where
$
\Delta(W,Z)_{\cdot\cdot s} = \mathbb{E}_{\cdot\cdot s}\left[\ell | W,Z\right]
    - \mathbb{E}_{\cdot\cdot 0}\left[\ell | W,Z\right].$
Setting $\phi_{\Y,0} = v_{\Y}(\emptyset)$, we denote the corresponding Shapley values by $\{\phi_{\Y,j}: j = 0,\cdots, m_2\}$.
Because this framework defines importance in terms of variance explained, it does not provide a detailed decomposition of the baseline shift.

\subsection{Candidate partial shifts}
\label{sec:candidate_regimes}
The previous section defines a general scoring scheme for quantifying the importance of \textit{any} hypothesized $s$-partial shift.
In this work, we consider partial shifts of the following forms.
We leave other forms of partial shifts to future work.

\textbf{$s$-partial conditional covariate shift}:
We hypothesize that $Z_{-s}$ is downstream of $Z_s$ and define $p_s(z|w)$ as $p_1(z_{s}|w)p_0(z_{-s}|z_s,w)$, as illustrated in Fig~\ref{fig:dag} right top.
A similar proposal was considered in \citet{Wu2021-dl}.
For $s=\emptyset$, note that $v_{\Z}(\emptyset) = 0$.

In the readmission example, such a shift would hypothesize that certain diagnosis codes in $Z_s$ are upstream of the others.
For instance, one could reasonably hypothesize that diagnosis codes that describe family history of disease are upstream of the other diagnosis codes.
Similarly, medical treatments/procedures are likely downstream of a patient's baseline characteristics and comorbidities.

\textbf{$s$-partial conditional outcome shift}: Shifting the conditional distribution of $Y$ only with respect to a variable subset $Z_s$ but not $Z_{-s}$ requires care.
We cannot simply split $Z$ into $Z_{s}$ and $Z_{-s}$, and define the distribution of $Y$ solely as a function of $Z_{s}$ and $W$, because defining $p_s(Y|W,Z):=p(Y|W,Z_s)$ for some $p$ generally implies that $p_s(Y|W,Z) \not \equiv p_0(Y|W,Z)$ even when $p_1(Y|W,Z) \equiv p_0(Y|W,Z)$.

Instead, we define an $s$-partial outcome shift based on models commonly used in model recalibration/revision \citep{Steyerberg2009-ze, Platt1999-lm}, where the modified risk is a function of the risk in the source domain \smash{$Q := q(W,Z) := p_0(Y=1|W,Z)$}, $W$, and $Z_s$.
In particular, we define the shift as
\begin{align}
& p_s(y|z,r,w):= p_1(y|z_{s},r,w)
\label{eq:outcome_det_regime}
\\
&= \int p_1(y|\tilde z_{-s},z_s,w)
p_1(\tilde{z}_{-s}|z_s,w, q(w,z_s,\tilde{z}_{-s})=r)
d\tilde{z}_{-s}.\nonumber
\end{align}
By defining the shifted outcome distribution solely as a function of $Q,W$, and $Z_s$, it encompasses the special scenario where there is no conditional outcome shift and eliminates any direct effect from $Z_{-s}$ to $Y$.
Note that $v_{\Y}(\emptyset)$ may be non-zero when the risk in the target domain is a recalibration (i.e. temperature-scaling) of the risk in the source domain \citep{Platt1999-lm, Guo2017-gy}.
For instance, there may be general environmental factors such that readmission risks in the target domain are uniformly lower.

In the readmission example, a partial outcome distribution shift may be due to, for instance, a hospital protocol where HF patients with given diagnoses $z_s$ (e.g. kidney failure) are scheduled for more frequent follow-ups than just having those diagnoses alone.
Thus readmission risks in the HF population can be viewed as a fine-tuning of readmission risks in the general patient population with respect to $z_s$.

\vspace{-0.2cm}
\section{Estimation and Inference}
\label{sec:theory}
The key estimands in this hierarchical attribution framework are the aggregate terms $\Lambda_\W,\Lambda_\Z$, and $\Lambda_\Y$ and the Shapley-based detailed terms $\phi_{\Z,j}$ and $\phi_{\Y,j}$ for $j = 0,\cdots, m_2$.
We now provide nonparametric estimators for these quantities as well as statistical inference procedures for constructing CIs.
Nearly all prior works for decomposing performance gaps have relied on \textit{plug-in} estimators, which substitute in estimates of the conditional means (also known as outcome models) or density ratios \citep{Sugiyama2007-pk}, which we collectively refer to as nuisance parameters.
For instance, given ML estimators $\hat{\mu}_{\cdot 10}$ and $\hat{\mu}_{\cdot 00}$ for the conditional means $\mu_{\cdot10}(w)=\mathbb{E}_{\cdot 10}[\ell|W]$ and $\mu_{\cdot00}(w)=\mathbb{E}_{\cdot 00}[\ell|W]$, respectively, the empirical mean of $\hat{\mu}_{\cdot 10} - \hat{\mu}_{\cdot 00}$ with respect to the target domain is a plug-in estimator for $\Lambda_\Z=\mathbb{E}_{1\cdot\cdot}[{\mu}_{\cdot 10} - {\mu}_{\cdot 00}]$.
However, because ML estimators typically converge to the true nuisance parameters at a rate slower than $n^{-1/2}$, plug-in estimators generally fail to be consistent at a rate of $n^{-1/2}$ and cannot be used to construct CIs \citep{Kennedy2022-to}.
Nevertheless, using tools from semiparametric inference theory (e.g. one-step correction and Neyman orthogonality), one \textit{can} derive debiased ML estimators that facilitate statistical inference \citep{Chernozhukov2018-dg,Tsiatis2006-ay}.
Here we derive debiased ML estimators for terms in the aggregate and detailed decompositions.
The detailed conditional outcome decomposition is particularly interesting, as its unique structure is not amenable to standard techniques for debiasing ML estimators.

For ease of exposition, theoretical results are presented for split-sample estimators; nonetheless, the results readily extend to cross-fitted estimators under standard convergence criteria \citep{Chernozhukov2018-dg, Kennedy2022-to}.
Let $\{(W_i^{(d)},Z_i^{(d)},Y_i^{(d)}): i = 1,\cdots, n_d\}$ denote $n_d$ independent and identically distributed (IID) observations from the source and target domains $d = 0$ and $1$, respectively.
Let a fixed fraction of the data be partitioned towards ``training'' ($\Tr$) and the remaining to ``evaluation'' ($\Tt$); let $n_{\Tt}$ be the number of observations in the evaluation partition.
Let $\mathbb{P}_d$ denote the expectation with respect to domain $d$ and $\mathbb{P}^{\Tt}_{d,n}$ denote the empirical average over observations in partition $\Tt$ from domain $d$.
All estimators are denoted using hat notation.
Proofs, detailed theoretical results, and psuedocode are provided in the Appendix.

\vspace{-0.2cm}
\subsection{Aggregate decomposition}
\label{sec:agg_est_inf}
The aggregate decomposition terms can be formulated as an average treatment effect, a well-studied estimand in causal inference, where domain $D$ corresponds to treatment.
As such, one can use augmented inverse probability weighting (AIPW) to define debiased ML estimators of the aggregate decomposition terms (e.g. \citep{kang2007demystifying}).
We review estimation and inference for these terms below.

\textbf{Estimation.}
Using the training data, estimate outcome models \smash{$\mu_{\cdot00}(W) = \mathbb{E}_{\cdot00}[\ell|W=w]$} and $\mu_{\cdot\cdot0}(W,Z) = \mathbb{E}[\ell|W,Z]$ and density ratio models $\pi_{100}(W) = p_1(W)/p_0(W)$ and $\pi_{110}(W,Z) = p_1(W,Z)/p_0(W,Z)$.
The debiased ML estimators for $\Lambda_\W, \Lambda_\Z, \Lambda_\Y$ are
\begin{align*}
\hat{\Lambda}_\W = &
\mathbb{P}^{\Tt}_{0,n}\left(\ell - \hat{\mu}_{\cdot00}(W)\right) \hat\pi_{100}(W) + \mathbb{P}^{\Tt}_{1,n} \hat{\mu}_{\cdot00}(W)
- \mathbb{P}^{\Tt}_{0,n} \ell\\
\hat{\Lambda}_\Z = & 
\mathbb{P}^{\Tt}_{0,n} \left(\ell - \hat{\mu}_{\cdot\cdot 0}(W,Z)\right) \hat\pi_{110}(W,Z) + \mathbb{P}^{\Tt}_{1,n} \hat{\mu}_{\cdot\cdot 0}(W,Z)\\
& - \mathbb{P}^{\Tt}_{0,n} \left(\ell - \hat{\mu}_{\cdot 00}(W)\right) \hat\pi_{100}(W) - \mathbb{P}^{\Tt}_{1,n} \hat{\mu}_{\cdot 00}(W)
\\
\hat{\Lambda}_\Y = & \mathbb{P}^{\Tt}_{1,n} \ell - \mathbb{P}^{\Tt}_{0,n} \left(\ell - \hat{\mu}_{\cdot\cdot0}(W,Z)\right) \hat\pi(W,Z) - \mathbb{P}^{\Tt}_{1,n} \hat{\mu}_{0}(W,Z)
\end{align*}

\textbf{Inference.} Assuming the estimators for the outcome and density ratio models converge at a fast enough rate, the AIPW estimators for the aggregate decomposition terms converge at the desired $o_p(n^{-1/2})$ rate.
\begin{theorem}
\label{thm:aggregate}
Suppose $\pi_{100}$ and $\pi_{110}$ are bounded; estimators $\hat{{\mu}}_{\cdot00}$, $\hat \pi_{\cdot \cdot 0}$, $\hat{\pi}_{100}$, and $\hat{\pi}_{110}$ are consistent; and
\begin{align*}
    &\mathbb{P}_0 \left(\hat{{\mu}}_{\cdot00} - {\mu}_{\cdot00}\right)\left(\hat{\pi}_{100} - {\pi}_{100}\right) = o_p(n^{-1/2})\\
    &\mathbb{P}_0 \left(\hat{\mu}_{\cdot\cdot0} - {\mu}_{\cdot\cdot0}\right)\left(\hat{\pi}_{110} - {\pi}_{110}\right) = o_p(n^{-1/2}).
\end{align*}
Then $\hat{\Lambda}_{\W}, \hat{\Lambda}_{\Z},$ and $\hat{\Lambda}_{\Y}$ are asymptotically linear estimators of their respective estimands.
\end{theorem}

\subsection{Detailed decomposition}

The Shapley-based detailed decomposition terms $\phi_{\Z,j}$ and $\phi_{\Y,j}$ for $j = 0,\cdots, m_2$ are an additive combination of the value functions that quantify the proportion of variability explained by $s$-partial shifts.
Below, we present novel estimators for values $\{v_\Y(s): s\subseteq \{1,\cdots,m_2\}\}$ for $s$-partial conditional outcome shifts, their theoretical properties, and computationally efficient estimation for their corresponding  Shapley values.
Due to space constraints, results pertaining to partial conditional covariate shifts are provided in Section~\ref{sec:cond_cov_s} of the Appendix.

\subsubsection{Value of $s$-partial outcome shifts}
\label{sec:cond_outcome_s}
Standard recipes for deriving asymptotically linear, nonparametric efficient estimators rely on pathwise differentiability of the estimand and analyzing its efficient influence function (EIF) \citep{Kennedy2022-to}.
However, $v_{\Y}(s)$ is not pathwise differentiable because it is a function of \eqref{eq:outcome_det_regime}, which conditions on the source risk $q(w,z)$ equalling some value $r$.
Taking the pathwise derivative of $v_{\Y}(s)$ requires taking a derivative of the indicator function $\mathbbm{1}\{q(w,z) = r\}$, which generally does not exist.
Given the difficulties in deriving an asymptotically normal estimator for $v_{\Y}(s)$, we propose estimating a close alternative that \textit{is} pathwise differentiable.

The idea is to replace $q$ in \eqref{eq:outcome_det_regime} with its binned variant $q_{\binned}(w,z) = \frac{1}{B} \lfloor q(w,z) B + \frac{1}{2} \rfloor$ for some $B\in \mathbb{Z}^+$, which discretizes outputs from $q$ into $B$ disjoint bins.
As long as $B$ is sufficiently high, the binned version of the estimand, denoted $v_{\Y,\binned}(s)$, is a close approximation to $v_{\Y}(s)$. (We use \smash{$B=20$} in the empirical analyses, which we believe to be sufficient in practice.)
The benefit of this binned variant is that the derivative of the indicator function $\mathbbm{1}\{q_{\binned}(w,z) = r\}$ is zero almost everywhere as long as observations with source risks exactly equal to a bin edge is measure zero.
More formally, we require the following:
\begin{condition}
    Let $\Xi$ be the set of $(W,Z)$ such that $q(W,Z)$ falls precisely on some bin edge but $q(W,Z)$ is not equal to zero or one.
    The set $\Xi$ is measure zero.
    \label{cond:bin_edge}
\end{condition}
Under this condition, $v_{\Y,\binned}(s)$ is pathwise differentiable and, using one-step correction, we derive a debiased ML estimator that has the unique form of a V-statistic (this follows from the integration over ``phantom'' $\tilde{z}_{-s}$ in \eqref{eq:outcome_det_regime}).
We represent V-statistics using the operator ${\mathbb{P}}_{1,n}^{\Tt} \tilde{\mathbb{P}}_{1,n}^{\Tt}$, which takes the average over all pairs of observations $O_i$ from the evaluation partition with replacement, i.e. $\frac{1}{n_{\Tt}^2} \sum_{i=1}^{n_{\Tt}} \sum_{j=1}^{n_{\Tt}} g(O_i, O_j)$ for some function $g$.
Calculation of this estimator and its theoretical properties are as follows.

\textbf{Estimation.} Using the training partition, estimate a density ratio model defined as $\pi(W,Z_s,Z_{-s},Q_\binned)=p_{1}(Z_{-s}|W,Z_s,q_\binned(W,Z)=Q_\binned)/p_{1}(Z_{-s})$ and the $s$-shifted outcome model $\mu_{\cdot \cdot s}(W,Z) = \mathbb{E}_{\cdot \cdot s}[\ell|W,Z]$.
The estimator for $v_{\Y,\binned}(s)$ is $\hat{v}_{\Y,\binned}(s) = \hat{v}_{\Y,n}^\num(s)/\hat{v}_{\Y,n}^{\denom}$ where $\hat{v}_{\Y,n}^\num(s)$ is
{\small
\begin{align}
& \mathbb{P}_{1,n}^{\Tt} \hat{\xi}_s(W,Z)^2 \label{eq:value_num_outcome}\\
    & + 2 \mathbb{P}_{1,n}^{\Tt} \hat{\xi}_s(W,Z) (\ell - \h \mu_{\cdot \cdot 1}(W,Z)) \nonumber\\
    & - 2\mathbb{P}_{1,n}^{\Tt} \tilde{\mathbb{P}}_{1,n}^{\Tt} \hat{\xi}_s(W,Z_s,\tilde{Z}_{-s})
    \ell( W,Z_s,\tilde{Z}_{-s}, Y)
    \h \pi(W,Z_s,\tilde{Z}_{-s}, Q_\binned) \nonumber\\
    & + 2\mathbb{P}_{1,n}^{\Tt} \tilde{\mathbb{P}}_{1,n}^{\Tt} \hat{\xi}_s(W,Z_s,\tilde{Z}_{-s})
    \h \mu_{\cdot \cdot s}(W,Z_s,\tilde{Z}_{-s})
    \h \pi(W,Z_s,\tilde{Z}_{-s}, Q_\binned) \nonumber
\end{align}
}
where $\hat{\xi}_s(W,Z) = \hat{\mu}_{\cdot \cdot 1}(W,Z) - \hat{\mu}_{\cdot \cdot s}(W,Z)$ and $\hat{v}_{\Y,n}^{\denom}$ is
{{\small
\begin{align}
    & \mathbb{P}_{1,n}^{\Tt} \left(\h \mu_{\cdot\cdot 1}(W,Z) - \h \mu_{\cdot\cdot 0}(W,Z)\right)^2 \label{eq:value_denom_outcome}\\
        & + 2 \mathbb{P}_{1,n}^{\Tt} \left(\h \mu_{\cdot\cdot 1}(W,Z) - \h \mu_{\cdot\cdot 0}(W,Z)\right) (\ell - \h \mu_{\cdot\cdot 1}(W,Z)) \nonumber\\
        & - 2 \mathbb{P}_{0,n}^{\Tt} \left(\h \mu_{\cdot\cdot 1}(W,Z) - \h \mu_{\cdot\cdot 0}(W,Z)\right) (\ell - \h \mu_{\cdot\cdot 0}(W,Z)) \h \pi_{110}(W,Z). \nonumber
\end{align}
}

\textbf{Inference.} This estimator is asymptotically linear assuming the nuisance estimators converge at a fast enough rate.
\begin{theorem}
Suppose Condition~\ref{cond:bin_edge} holds.
For variable subset $s$, suppose $\pi(W,Z_s,Z_{-s},Q_{\binned}),\pi_{110}$ are bounded; $v_{\Y}^{\denom}(\emptyset) > 0$; estimator $\hat{\pi}$ is consistent; estimators $\hat \mu_{\cdot \cdot 0}, \hat \mu_{\cdot \cdot 1}$ and $\hat \mu_{\cdot \cdot s}$ converge at an $o_p(n^{-1/4})$ rate, and
    \begin{align}
    \mathbb{P}_1 (\hat{q}_{\binned} - {q}_{\binned})^2 & = o_p(n^{-1}) \label{eq:bin_conv}
    \\
\mathbb{P}_1 (\mu_{\cdot \cdot s} - \hat{\mu}_{\cdot \cdot s})(\pi - \hat{\pi}) &= o_p(n^{-1/2})
        \nonumber\\
\mathbb{P}_0 (\mu_{\cdot \cdot 0} - \hat{\mu}_{\cdot \cdot 0})(\pi_{110} - \h \pi_{110}) &= o_p(n^{-1/2}) \nonumber
\end{align}
    Then the estimator $\hat{v}_{\Y,\binned}(s)$ is asymptotically normal centered at the estimand ${v}_{\Y,\binned}(s)$.
\label{theorem:conditional_outcome_detailed}
\end{theorem}
Convergence rate of $o_p(n^{-1/2})$ can be achieved by ML estimators in a wide variety of conditions, and such assumptions are commonly used to construct debiased ML estimators.
The additional requirement in \eqref{eq:bin_conv} that $\hat q_{\binned}$ converges at a $o_p(n^{-1})$ rate is new, but fast or even super-fast convergence rates of \textit{binned} risks is achievable under suitable margin conditions \citep{Audibert2007-wx}.

\subsubsection{Shapley values}
\label{sec:shapley_inf}
Calculating the exact Shapley value is computationally intractable as it involves an exponential number of terms.
Nevertheless, \citet{Williamson2020-lx} showed that calculating the exact Shapley value for variable importance measures is not necessary when the importance measures are estimated with uncertainty.
One can instead estimate the Shapley values by sampling variable subsets and inflate the corresponding CIs to reflect the additional uncertainty due to subset sampling.
As long as the number of subsets scales linearly or super-linearly in $n$, one can show that this additional uncertainty will not be larger than that due to estimation of subset values themselves.
Here we follow the same procedure to estimate and construct CIs for the detailed decomposition Shapley values (see Algorithm~\ref{algo:detailed_decomp}).

\begin{figure}[t]
    \centering
    \vspace{-0.2cm}
    \includegraphics[width=0.15\textwidth]{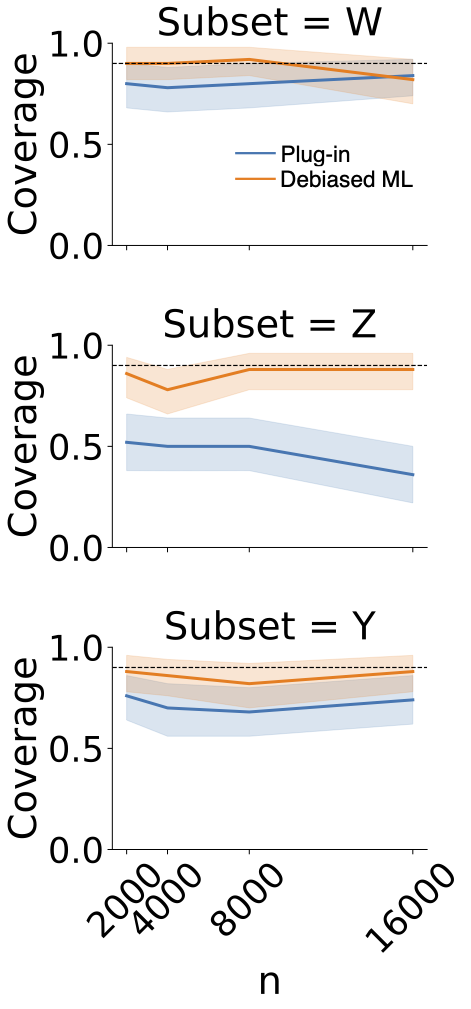}
    \includegraphics[width=0.15\textwidth]{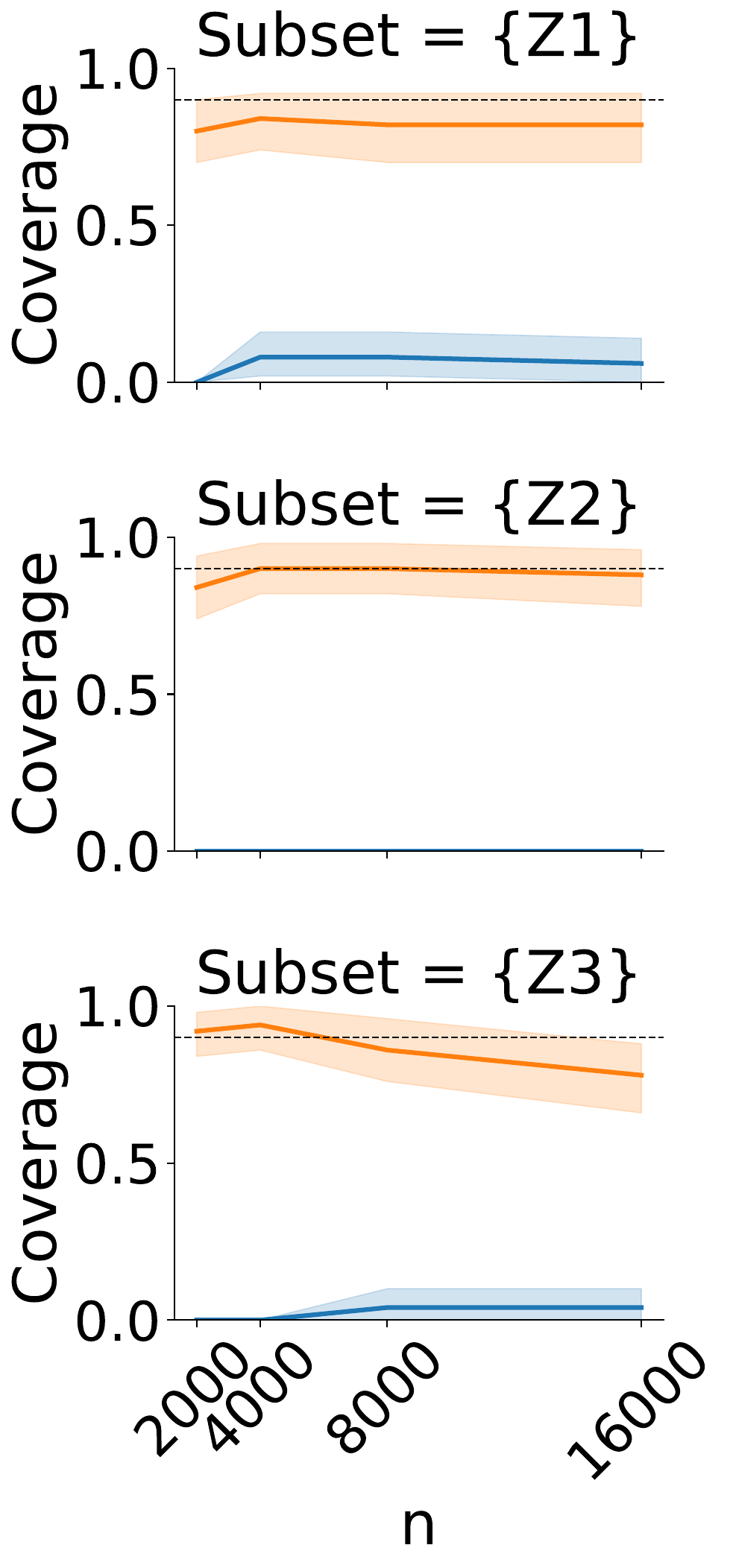}
    \includegraphics[width=0.15\textwidth]{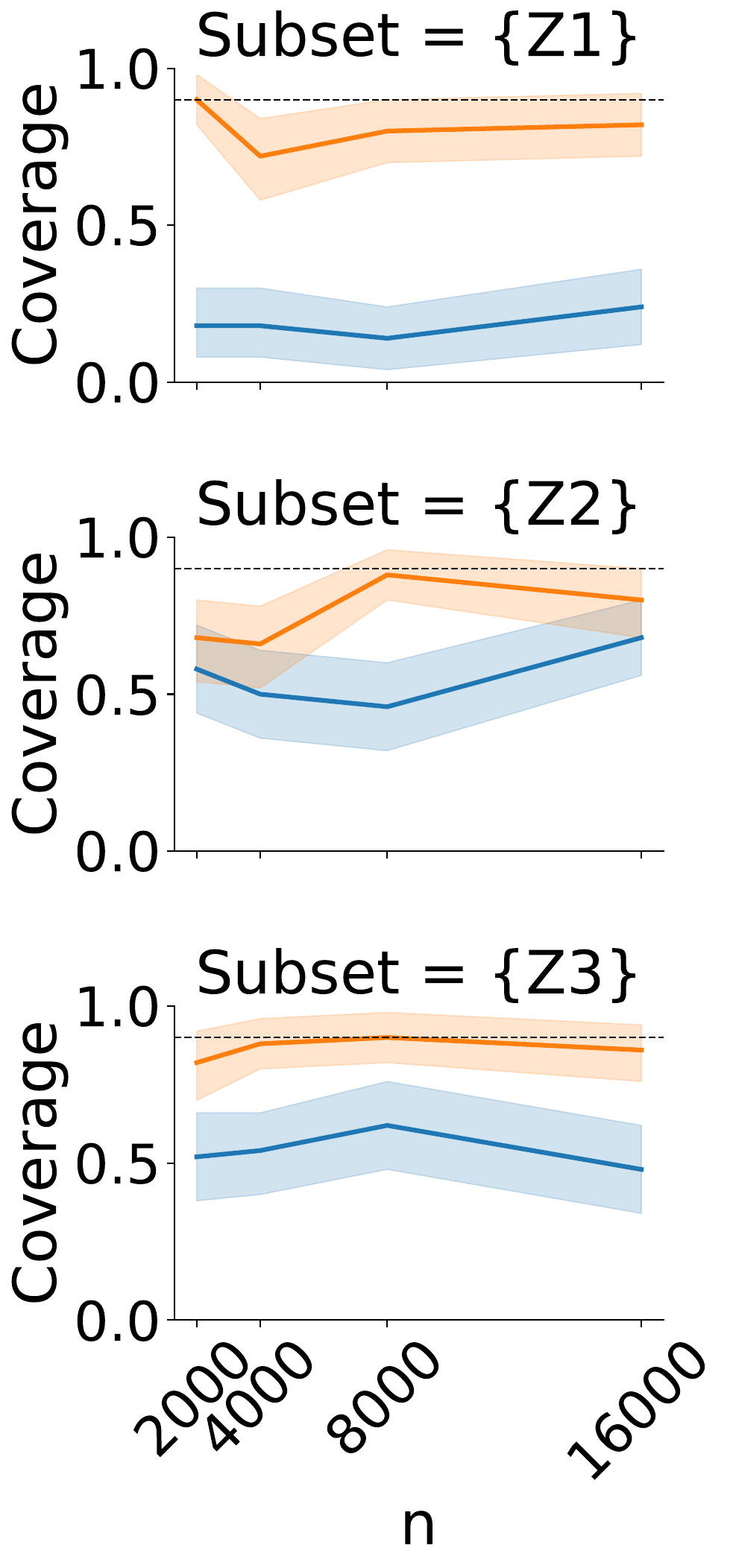}
    \vspace{-0.4cm}
    \caption{
    Coverage rates of 90\% CIs for aggregate decomposition terms (left) and value of $s$-partial shifts for the conditional covariate (middle) and outcome shifts (right) across dataset sizes $n$.
    }
    \vspace{-0.8cm}
    \label{fig:simulation_cov}
\end{figure}

\vspace{-0.2cm}
\section{Simulation}
\label{sec:simulations}
We now present simulations that
verify the theoretical results by showing that the proposed procedure achieves the desired coverage rates (Sec~\ref{sec:coverage})
and illustrate how the proposed method provides more intuitive explanations of performance gaps (Sec~\ref{sec:compare}).
In all empirical analyses, performance of the ML algorithm is quantified in terms of 0-1 accuracy.
Below, we briefly describe the simulation settings; full details are provided in Sec~\ref{sec:simulation_details} in the Appendix.
\vspace{-0.2cm}
\subsection{Verifying theoretical properties}
\label{sec:coverage}
We first verify that the inference procedures for the decomposition terms have CIs with coverage close to their nominal rate.
We check the coverage of the aggregate decomposition as well as the value of $s$-partial conditional covariate and partial conditional outcome shifts for $s = \{Z_1\}, \{Z_2\}, \{Z_3\}$.
$(W,Z_1,Z_2,Z_3)$ are sampled from independent normal distributions with different means in source and target, while $Y$ is simulated from logistic regression models with different coefficients.
CIs for the debiased ML estimator converge to the nominal 90\% coverage rate with increasing sample size, whereas those for the na\"ive plug-in estimator do not (Fig~\ref{fig:simulation_cov}).

\begin{figure} \centering
    \textit{(a) Conditional covariate}
    \includegraphics[scale=0.33]{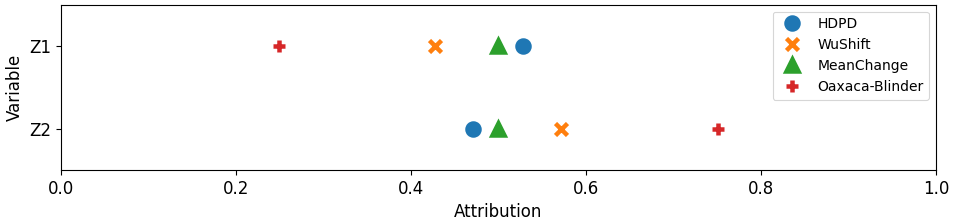}
    
    \textit{(b) Conditional outcome}
    
    \includegraphics[scale=0.33]{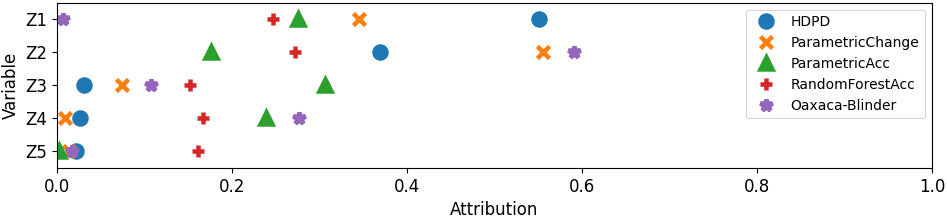}
{\small
\begin{tabular}{lrrr}
Method & Acc-1 & Acc-2 & Acc-3 \\
\midrule
ParametricChange & 0.78 & \textbf{0.80} & \textbf{0.81} \\
ParametricAcc & 0.75 & \textbf{0.80} & 0.80 \\
RandomForestAcc & 0.78 & \textbf{0.80} & 0.80 \\
OaxacaBlinder & 0.78 & 0.78 & 0.80 \\
Proposed & \textbf{0.80} & \textbf{0.80} & \textbf{0.81} \\
\end{tabular}
    }
    \caption{
    Comparison of variable importance reported by proposed method HDPD versus existing methods for conditional covariate (a) and conditional outcome (b) terms. 
}
    \vspace{-0.5cm}
    \label{fig:comparators}
\end{figure}

\subsection{Comparing explanations}
\label{sec:compare}

We now compare the proposed definitions for the detailed decomposition with existing methods.
For the detailed conditional covariate decomposition, the comparators are:
\begin{itemize}
    \item \texttt{MeanChange} Tests for a difference in means for each variable. Defines importance as $1-$ p-value.
    \item \texttt{Oaxaca-Blinder}: Fits a linear model of the logit-transformed expected loss with respect to $Z$ in the source domain. Defines importance of $Z_i$ as its coefficient multiplied by the difference in the mean of $Z_i$ \citep{blinder1973wage}.
    \item \texttt{WuShift} \citep{Wu2021-dl}: Defines importance of subset $s$ as change in \textit{overall} performance due to $s$-partial conditional covariate shifts. Applies Shapley framework to obtain VIs.
\end{itemize}
For the detailed decomposition of the performance gap due to shifts in the outcome, we compare against:
\begin{itemize}
    \item \texttt{ParametricChange}: Fits a logistic model for $Y$ with interaction terms between domain and $Z$. Defines importance of $Z_i$ as the coefficient of its interaction term.
    \item \texttt{ParametricAcc}: Same as \texttt{ParametricChange} but models the 0-1 loss rather than $Y$.
\item \texttt{RandomForestAcc}: Compares VI of RFs trained on data from both domains with input features $D$, $Z$, and $W$ to predict the 0-1 loss.
    \item \texttt{Oaxaca-Blinder}: Fits linear models for the logit-transformed expected loss in each domain. Defines importance of $Z_i$ as its mean in the target domain multiplied by the difference in its coefficients across domains.
\end{itemize}
Although the proposed method agrees on important features with these other methods in certain cases, there are important situations where the methods differ as highlighted below.

\textbf{Conditional covariate.} (Fig~\ref{fig:comparators}a) We simulate $(W, Z_1)$ from a standard normal distribution, $Z_2$ from a mixture of two Gaussians whose means depend on the value of $Z_1$ (i.e. $Z_1 \rightarrow Z_2$), and $Y$ from a logistic regression model depending on $(W,Z_1,Z_2)$.
We induce a shift from the source domain to the target domain by shifting only the distribution of $Z_1$, so that $p_1(Z|W) = p_0(Z_{2}|Z_1,W)p_1(Z_1|W)$.
Only the proposed estimator correctly recovers that $Z_1$ is more important than $Z_2$, because this shift explains all the variation in performance gaps across strata $W$.
The other methods incorrectly assign higher importance to $Z_2$ because they simply assess the performance change due to hypothesized $s$-partial shifts but do not check if the hypothesized shifts are good explanations in the first place.

\textbf{Conditional outcome.} (Fig~\ref{fig:comparators}b)  $W$ and $Z \in \mathbb{R}^5$ are simulated from the same distribution in both domains. $Y$ is generated from a logistic regression model with coefficients for $(W, Z_1,\cdots, Z_5)$ as $(0.2,0.4,2,0.25,0.1,0.1)$ in the source and $(0.2,-0.4,0.8,0.1,0.1,0.1)$ in the target.
Interestingly, none of the methods have the same ranking of the features.
\texttt{ParametricChange} identifies $Z_2$ as having the largest shift on the logit scale, but this does not mean that it is the most important explanation for changes in the loss. 
According to our decomposition framework, $Z_1$ is actually the most important for explaining changes in model performance due to outcome shifts.
\texttt{Oaxaca-Blinder} and \texttt{ParametricAcc} have odd behavior---\texttt{Oaxaca-Blinder} assigns $Z_1$ the lowest importance and \texttt{ParametricAcc} assigns $Z_3$ the highest importance), because the models are misspecified.
The VI from \texttt{RandomForestAcc} is also difficult to interpret because it measures which variables are good predictors of performance, not performance shift.

\begin{figure}[t!]
    \centering
    \textit{(a) Readmission risk prediction}
    
    \includegraphics[scale=0.32]{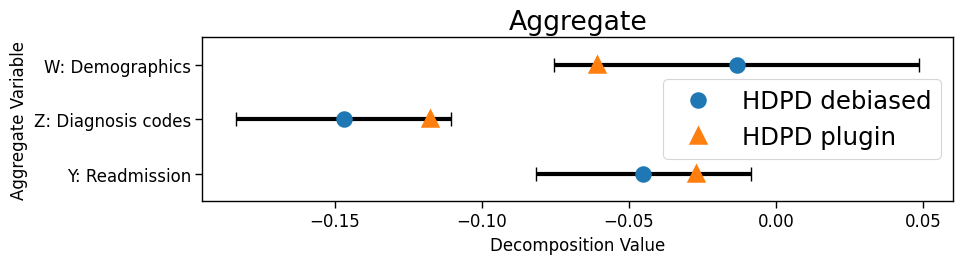}
    \includegraphics[scale=0.32]{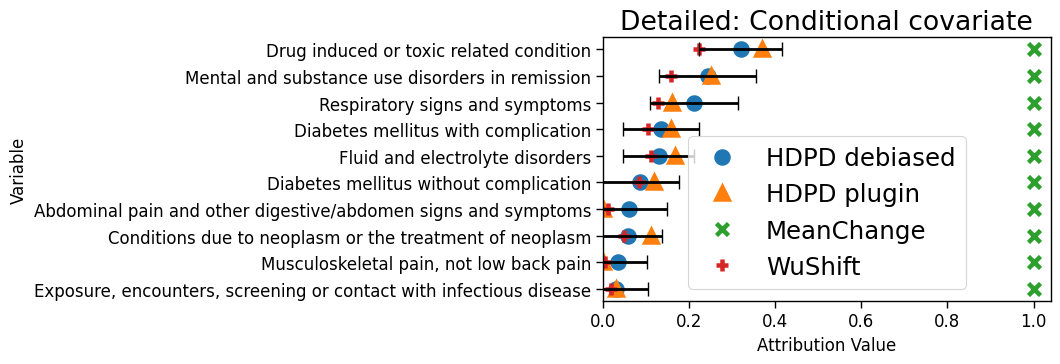}

    \textit{(b) Insurance coverage prediction}

    \includegraphics[scale=0.32]{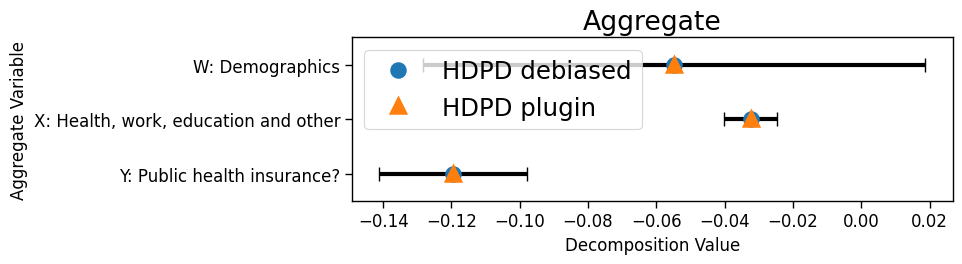}
    \includegraphics[scale=0.32]{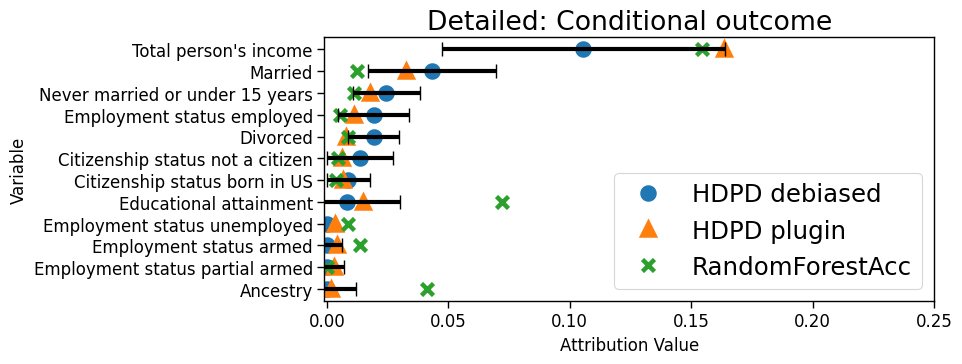}
    
    \vspace{-0.3cm}
    \caption{
    Aggregate and detailed decompositions for the performance gaps of (a) a model predicting readmission risk across two patient populations (General$\rightarrow$Heart Failure) and (b) a model predicting insurance coverage across US states (NE$\rightarrow$LA).
    A subset of VI estimates is shown; see full list in Sec~\ref{sec:data_details} in the Appendix.
    }
    \vspace{-0.5cm}
    \label{fig:casestudy}
\end{figure}

Perhaps a more objective evaluation is to compare the utility of the different explanations.
To this end, we define a targeted algorithmic modification as one where the source risk is revised with respect to a subset of features by fitting an ML algorithm with input features as $Q$, $W$, and $Z_s$ on the target domain.
Comparing the performance of the targeted algorithmic modifications that take in the top $k$ features from each explanation method, we find that model revisions based on the proposed method achieve the highest performance for $k=1$ to $3$.
\section{Real-world data case studies}
We now analyze two datasets with naturally-occurring shifts.

\textbf{Hospital readmission.} Using electronic health record data from a large safety-net hospital, we analyzed performance of a Gradient Boosted Tree (GBT) trained on the general patient population (source) to predict 30-day readmission risk but applied to patients diagnosed with HF (target).
Features include  4 demographic variables ($W$) and 16 diagnosis codes ($Z$).
Each domain supplied $n=3750$ observations. 

Model accuracy drops from 70\% to 53\%.
From the aggregate decompositions (Fig~\ref{fig:casestudy}a), we observe that the drop is mainly due to covariate shift.
If one performed the standard check to see which variables significantly changed in their mean value (\texttt{MeanChange}), then one would find a significant shift in nearly \textit{every} variable.
Little support is offered to identify main drivers of the performance drop.
In contrast, the detailed decomposition from the proposed framework estimates diagnoses ``Drug-induced or toxic-related condition'' and ``mental and substance use disorder in remission'' as having the highest estimated contributions to the conditional covariate shift, and most other variables having little to no contribution.
Upon discussion with clinicians from this hospital, differences in the top two diagnoses may be explained by (i) substance use being a major cause of HF at this hospital, with over eighty percent of its HF patients reporting current or prior substance use, and (ii) substance use and mental health disorders often occurring simultaneously in this HF patient population.
Based on these findings, closing the performance gap may require a mixture of both operational and algorithmic interventions.
Finally, CIs from the debiased ML procedure provide valuable information on the uncertainty of the estimates and highlight, for instance, that more data is necessary to determine the true ordering between the top two features.
In contrast, existing methods do not provide (asymptotically valid) CIs.

\textbf{ACS Public Coverage.} 
We analyze a neural network trained to predict whether a person has public health insurance using data from Nebraska in the American Community Survey (source, $n=3000$), applied to data from Louisiana (target, $n=6000$).
Baseline variables include 3 demographics (sex, age, race), and covariates $Z$ include 31 variables related to health conditions, employment, marital status, citizenship status, and education.

Model accuracy drops from 84\% to 66\% across the two states.
The main driver is the shift in the outcome distribution per the aggregate decomposition (Fig~\ref{fig:casestudy}b) and the most important contributor to the outcome shift is annual income, perhaps due to differences in cost of living across the two states.
Income is significantly more important than all the other variables; the ranking between the remaining variables is unclear.
In comparing the performance of targeted model revisions with respect to the top variables from each explanation method, we find that model revisions based on top variables identified by the proposed procedure lead to AUCs that are better or as good as those based on \texttt{RandomForestAcc} (Table~\ref{tab:acs_auc} in the Appendix).

\vspace{-0.2cm}
\section{Discussion}

ML algorithms regularly encounter distribution shifts in practice, leading to drops in performance.
We present a novel framework that helps ML developers and deployment teams build a more nuanced understanding of the shifts.
Compared to past work, the approach is nonparametric, does not require fine-grained knowledge of the causal relationship between variables, and quantifies the uncertainty of the estimates by constructing confidence intervals. We present case studies on real-world datasets to demonstrate the use of the framework for understanding and guiding interventions to reduce performance drops.

Extensions of this work include relaxing the assumption of overlapping support of covariates such as by restricting to the common support \citep{Cai2023-ov}, allowing for decompositions of more complex measures of model performance such as AUC, and analyzing other factorizations of the data distribution (e.g. label or prior shifts \citep{kouw2019domain}).
For unstructured data (e.g. image and text), the current framework can be applied to low-dimensional embeddings or by extracting interpretable concepts \citep{kim2018tcav}; more work is needed to extend this framework to directly analyze unstructured data.
Finally, the focus of this work is to interpret performance gaps.
Future work may extend ideas in this work to design optimal interventions for closing the performance gap.

\subsection*{Impact statement}
This work presents a method for understanding failures of ML algorithms when they are deployed in settings or populations different from the ones in development datasets. Therefore, the work can be used to suggest ways of improving the algorithms or mitigating their harms. The method is generally applicable to tabular data settings for any classification algorithm, hence, it can potentially be applied across multiple domains where ML is used including medicine, finance, and online commerce.

\section*{Acknowledgments}
We would like to thank Lucas Zier, Avni Kothari, Berkman Sahiner, Nicholas Petrick, Gene Pennello, Mi-Ok Kim, and Romain Pirracchio for their invaluable feedback on the work.
This work was funded through a Patient-Centered Outcomes Research Institute® (PCORI®) Award (ME-2022C1-25619).
The views presented in this work are solely the responsibility of the author(s) and do not necessarily represent the views of the PCORI®, its Board of Governors or Methodology Committee, and the Food and Drug Administration.

\bibliography{main}
\bibliographystyle{icml2024}

\newpage
\appendix
\onecolumn
\section{Contents of the Appendix}
\begin{itemize}
\item Table~\ref{tab:comparison} summarizes the comparison with prior work.

\item Table~\ref{tab:notation} collects all the notation used for reference.

\item Algorithms~\ref{algo:agg} and \ref{algo:detailed_decomp} provides the steps required for computing the aggregate and detailed decomposition respectively. Detailed decompositions require computing the value of $s$-partial conditional outcome and conditional covariate shifts which is described in Algorithms~\ref{algo:value_cond_outcome} and \ref{algo:value_cond_cov}.

\item Section~\ref{sec:cov_inference} describes the estimation and inference for detailed decomposition of conditional covariate shift. 

\item Section~\ref{sec:proofs} provides the derivations of the results.

\item Sections~\ref{sec:implement} and \ref{sec:simulation_details} describe the implementation and simulation details.

\item Section~\ref{sec:data_details} provides additional details on the two real world datasets and results.
\end{itemize}

\begin{table}[htbp!]
    \centering
    \begin{tabular}{p{2.8cm}|p{1.5cm}|p{2.2cm}|p{2.2cm}|p{2cm}|p{1.5cm}|p{2cm}}
        \toprule
        Papers &  Aggregate decomp. & \multicolumn{2}{|c|}{Detailed decomp.} & Does not require knowing causal DAG & Confidence intervals & Nonparametric\\
        & & Cond. covariate & Cond. outcome & & &\\
        \midrule
        \citet{zhang2023why} & $\checkmark$ & $\checkmark$ & & & & $\checkmark$\\
        \citet{Cai2023-ov} & $\checkmark$ & & & $\checkmark$ & $\checkmark$ & $\checkmark$\\
        \citet{Wu2021-dl} &  & $\checkmark$ & & $\checkmark$ & & $\checkmark$\\
        \citet{Liu2023-gs} & $\checkmark$ & & $\checkmark$ & $\checkmark$ & & $\checkmark$\\
        \citet{Dodd2003-fb} & & & $\checkmark$ & $\checkmark$ & $\checkmark$\\
        \citet{oaxaca1973differentials}, \citet{blinder1973wage} & $\checkmark$ & $\checkmark$ & $\checkmark$ & $\checkmark$ & $\checkmark$ & \\
        \midrule
        HDPD (this paper) & $\checkmark$ & $\checkmark$ & $\checkmark$ & $\checkmark$ & $\checkmark$ & $\checkmark$\\
        \bottomrule
    \end{tabular}
    \caption{Comparison to prior works that decompose performance change.}
    \label{tab:comparison}
\end{table}

\begin{table}[htbp!]
    \centering
    \begin{tabular}{p{3cm}p{10cm}}
        \toprule
        Symbol & Meaning \\
        \midrule
        $W,Z,Y$ & Variables: Baseline, Conditional covariates, Outcome\\
        $f$ & Prediction model being analyzed \\
        $\ell(W,Z,Y)$ or $\ell$ & Loss function e.g. 0-1 loss\\
        $D=0$ and $D=1$ & Indicators for source and target domain\\
        $p_0,p_1$ & Probability density (or mass) function for the two domains $D=0,1$\\
        $\E_{D_\W D_\Z D_\Y}$ & Expectation over the distribution $p_{D_\W}(W)p_{D_\Z}(Z|W)p_{D_\Y}(Y|W,Z)$\\ 
        $Q:=q(W,Z)$ & Source domain risk at $W,Z$, i.e. $p_0(Y=1|W,Z)$\\
        $\Tr$ and $\Tt$ & Training dataset used to fit models and evaluation dataset used to compute decompositions\\
        $\phi_{\Z,j}$ and $\phi_{\Y,j}$ & Shapley values for variable $j$ in the detailed decomposition of conditional covariate and outcome shifts\\
        $v_{\Z}(s)$ and $v_{\Y}(s)$ & Value of a subset $s$ for $s$-partial conditional covariate shift and $s$-partial outcome shift\\
        $v_{\cdot}^{\num}(s)$ and $v_{\cdot}^{\denom}(s)$ & Numerator and denominator of the ratio defined in the value of a subset\\
        Models $\mu_{\cdot}$ & Outcome models for the conditional expectation of the loss across different settings\\
        Models $\pi_{\cdot}$ & Density ratio models for feature densities across datasets\\
        $\mathbb{P}$ & Notation for expectation\\
        $\mathbb{P}_{0,n}^{\Tt}$ and $\mathbb{P}_{1,n}^{\Tt}$ & sample average over source and target data in the evaluation dataset\\
        $\psi(d,w,z,y)$ & Influence function defined in the linear approximation of an estimand, see e.g. \eqref{eq:linearization}\\
        \bottomrule
    \end{tabular}
    \caption{Notation}
    \label{tab:notation}
\end{table}

\begin{algorithm}[htbp!]
    \SetKwFunction{isOddNumber}{isOddNumber}
\LinesNumbered

    \KwIn{Source and target data $\{(W_i^{(d)},Z_i^{(d)},Y_i^{(d)})\}_{i=1}^{n_{d}}$ for $d\in\{0,1\}$, loss function $\ell(W,Z,Y;f)$.}
    \KwOut{Performance change due to baseline, conditional covariate, and conditional outcome shifts $\Lambda_{\W},\Lambda_{\Z},\Lambda_{\Y}$.}

    Split source and target data into training $\Tr$ and evaluation $\Tt$ partitions. Let $n^{\Tt}$ be the total number of data points in the $\Tt$ partition.\\

    Fit nuisance parameters $\eta_{\W},\eta_{\Z},\eta_{\Y}$, defined in Section~\ref{sec:result_agg}, on the $\Tr$ partition as outlined in Section~\ref{sec:implement_agg}.\\

    Estimate $\Lambda_{\W},\Lambda_{\Z},\Lambda_{\Y}$ using fitted nuisance parameters on the $\Tt$ partitions following the equations in Section~\ref{sec:agg_est_inf}.\\

    Estimate variance of influence functions $\psi_{\W}(d,w,z,y; \hat{\eta}_{\N}), \psi_{\Z}(d,w,z,y; \hat{\eta}_{\N})$, and $\psi_{\Y}(d,w,z,y; \hat{\eta}_{\N})$ as defined in \eqref{eq:influence_agg_baseline}, \eqref{eq:influence_agg_cov}, and \eqref{eq:influence_agg_outcome}, respectively.\\

    Compute $\alpha$-level confidence intervals as $\h \Lambda_{\N}\pm z_{1-\alpha/2}\sqrt{\widehat{var}(\psi_{\N}(d,w,z,y;\h \eta_{\N}))/n^{\Tt}}$ for $\N\in \{\W,\Z,\Y\}$, where $z$ is the inverse CDF of the standard normal distribution.

    \KwRet{$\h \Lambda_{\W},\h \Lambda_{\Z},\h \Lambda_{\Y}$ and confidence intervals}
    \caption{Aggregate decompositions into baseline, conditional covariate, and conditional outcome shifts}
    \label{algo:agg}
\end{algorithm}

\begin{algorithm}[htbp!]
    \setcounter{AlgoLine}{0}
    \SetKwFunction{isOddNumber}{isOddNumber}
\LinesNumbered

    \KwIn{Training $\Tr$ and evaluation $\Tt$ partitions of source and target data, subset of variables $s$.}
    \KwOut{Value for $s$-partial conditional outcome shift for subset $s$.}

    Fit nuisance parameters $\eta_{\Y,s}^{\num},\eta_{\Y}^{\denom}$, defined in Sections~\ref{sec:result_detailed_outcome}, on the $\Tr$ partitions as outlined in \ref{sec:implement_detailed_outcome}.\\

    Estimate $v_{\Y}(s)$ by $\h v_{\Y}^{\num}(s)/\h v_{\Y}^{\denom}$ where $\h v_{\Y}^{\num}(s)$ is estimated using \eqref{eq:value_num_outcome} and $\h v_{\Y}^{\denom}$ is estimated using \eqref{eq:value_denom_outcome} on the $\Tt$ partition.\\

    Estimate variance of influence function $\psi_{\Y,\binned,s}(d,w,z,y;\h \eta_{\Y,s}^{\num},\h \eta_{\Y}^{\denom})$ as defined in \eqref{eq:influence_cond_outcome}.\\

    Compute $\alpha$-level confidence interval as $\h v_{\Y}(s)\pm z_{1-\alpha/2}\sqrt{\widehat{var}(\psi_{\Y,\binned,s}(d,w,z,y;\h \eta_{\Y,s}^{\num},\h \eta_{\Y}^{\denom}))/n^{\Tt}}$.

    \KwRet{$\h v_{\Y}(s)$ and confidence interval}
    \caption{\textsc{ValueConditionalOutcome(s)}: Value for $s$-partial conditional outcome shift for a subset $s$}
    \label{algo:value_cond_outcome}
\end{algorithm}

\begin{algorithm}[htbp!]
    \setcounter{AlgoLine}{0}
    \SetKwFunction{isOddNumber}{isOddNumber}
\LinesNumbered

    \KwIn{Training $\Tr$ and evaluation $\Tt$ partitions of source and target data, subset of variables $s$.}
    \KwOut{Value for $s$-partial conditional covariate shift for subset $s$.}

    Fit nuisance parameters $\eta_{\Z,s}^{\num}$, defined in Sections~\ref{sec:result_detailed_cov}, on the $\Tr$ partition, as outlined in \ref{sec:implement_detailed_cov}.\\

    Estimate $v_{\Z}(s)$ by $\h v_{\Z}^{\num}(s)/\h v_{\Z}^{\num}(\emptyset)$ using \eqref{eq:value_num_cov} on the $\Tt$ partition.\\

    Estimate variance of influence function $\psi_{\Z,s}(d,w,z,y;\h \eta_{\Z,s}^{\num})$ as defined in \eqref{eq:influence_detailed_cov}.\\

    Compute $\alpha$-level confidence interval as $\h v_{\Z}(s)\pm z_{1-\alpha/2}\sqrt{\widehat{var}(\psi_{\Z,s}(d,w,z,y;\h \eta_{\Z,s}^{\num}))/n^{\Tt}}$.

    \KwRet{$\h v_{\Z}(s)$ and confidence interval}
    \caption{\textsc{ValueConditionalCovariate(s)}: Value for $s$-partial conditional covariate shift for a subset $s$}
    \label{algo:value_cond_cov}
\end{algorithm}

\begin{algorithm}[htbp!]
    \setcounter{AlgoLine}{0}
    \SetKwFunction{isOddNumber}{isOddNumber}
\LinesNumbered

    \KwIn{Source and target data $\{(W_i^{(d)},Z_i^{(d)},Y_i^{(d)})\}_{i=1}^{n_{d}}$ for $d\in\{0,1\}$, loss function $\ell(W,Z,Y;f)$, $\gamma \in \mathbb{R}^+$.}
    \KwOut{Detailed decomposition for conditional outcome or covariate shift, $\{\phi_{\Y,j}: j=0,\cdots, m_2\}$ or $\{\phi_{\Z,j}: j=1,\cdots, m_2\}$.}

    Split source and target data into training $\Tr$ and evaluation $\Tt$ partitions. Let $n^{\Tt}$ be the total number of data points in the $\Tt$ partition.\\

    Subsample $\lfloor \gamma n^{\Tt} \rfloor $ subsets from $\mathcal{Z} = \{1,\cdots, m_2\}$ with respect to Shapley weights, including $\emptyset$ and $\mathcal{Z}$, denoted $s_1,\cdots, s_k$.\\

    Estimate $v_{\Y}(s)\gets$ \textsc{ValueConditionalOutcome(s)} and $v_{\Z}(s)\gets$ \textsc{ValueConditionalCovariate(s)} for $s\in s_1,\cdots, s_k$.\\

    Get estimated Shapley values $\{\phi_{\Y,j}\}$ and $\{\phi_{\Z,j}\}$ by solving constrained linear regression problems in (7) in \citet{Williamson2020-lx} with value functions $v_{\Y}(s)$ and $v_{\Z}(s)$, respectively.\\

    Compute confidence intervals based on the influence functions defined in Theorem 1 in \citet{Williamson2020-lx}.\\

    \KwRet{Shapley values  $\{\phi_{\Y,j}: j=0,\dots, m_2\}$ and $\{\phi_{\Z,j}: j=1,\dots, m_2\}$ and confidence intervals}
    \caption{Detailed decomposition for conditional outcome and covariate shift}
    \label{algo:detailed_decomp}
\end{algorithm}

\section{Estimation and Inference}
\label{sec:cov_inference}

\subsection{Value of $s$-partial conditional covariate shifts}
\label{sec:cond_cov_s}
\textbf{Estimation.} Using the training partition, estimate the density ratio $\pi_{1s0}(z_s, w) = p_1(z_{s},w)/p_0(z_{s},w)$ and the outcome models
\begin{align}
\mu_{\cdot 0_{-s} 0}(z_s,w)&=\mathbb{E}_{\cdot 00}[\ell|z_{s},w] =\int  \int \ell p_0(y|w,z)p_0(z_{-s}|z_s, w)dy dz_{-s}\label{eq:mu_0ms0_cond_cov}\\
\mu_{\cdot 10}(w)&=\E_{\cdot 1 0}[\ell|w]\label{eq:mu_dot10_cond_cov}\\
\mu_{\cdot s0}(w)&= \E_{\cdot s 0}[\ell|w] = \int \int \int \ell p_0(y|w,z)p_0(z_{-s}|z_s, w)p_1(z_s|w)dy dz_{-s}dz_{s}\label{eq:mu_dots0_cond_cov},
\end{align}
in addition to the other nuisance models previously mentioned.
We propose the estimator $\hat{v}_{\Z}(s) = \hat{v}_{\Z}^{\num}(s)/\hat{v}^{\denom}_{\Z}$, where
\begin{align}
\begin{split}
\hat{v}_{\Z}^{\num}(s) := & \mathbb{P}_{1,n}^{\Tt} (\h\mu_{\cdot s 0}(W) - \h\mu_{\cdot 1 0}(W))^2\\
    & + 2\mathbb{P}_{0,n}^{\Tt} (\h\mu_{\cdot s 0}(W) - \h\mu_{\cdot 1 0}(W)) (\ell - \h{\mu}_{\cdot0_{-s}0}(W,Z_s))\h{\pi}_{1s0}(W,Z_s)\\
    & - 2\mathbb{P}_{0,n}^{\Tt} (\h\mu_{\cdot s 0}(W) - \h\mu_{\cdot 1 0}(W)) (\ell - \h{\mu}_{\cdot\cdot 0}(W,Z))\h{\pi}_{110}(W,Z)\\        
    &+ 2\mathbb{P}_{1,n}^{\Tt} (\h\mu_{\cdot s 0}(W) - \h\mu_{\cdot 1 0}(W)) (\h{\mu}_{\cdot0_{-s}0}(W,Z_s) - \h{\mu}_{\cdot s0}(W))\\
    &- 2\mathbb{P}_{1,n}^{\Tt} (\h\mu_{\cdot s 0}(W) - \h\mu_{\cdot 1 0}(W)) (\h{\mu}_{\cdot \cdot 0}(W,Z) - \h{\mu}_{\cdot 10}(W))
\end{split}
\label{eq:value_num_cov}
\end{align}
and $\hat{v}_{\Z}^{\denom} := \hat{v}_{\Z}^{\num}(\emptyset)$.

\textbf{Inference.} The estimator is asymptotically normal as long as the outcome and density ratio models are estimated at a fast enough rate defined formally as follows.

\begin{condition}
    \label{cond:condcov}
    For variable subset $s$, suppose the following holds
    \begin{itemize}
        \item $\mathbb{P}_0 (\mu_{001}(W)-\mu_{01}(W))^2$ is bounded
        \item $\mathbb{P}_0(\mu_{\cdot 0_{-s}0}(Z_s,W) - \h{\mu}_{\cdot 0_{-s}0}(Z_s,W))(\h{\pi}_{1s0}(Z_s,W)-\pi_{1s0}(Z_s,W)) = o_p(n^{-1/2})$
        \item $\mathbb{P}_0(\mu_{\cdot\cdot0}(W,Z) - \h{\mu}_{\cdot\cdot0}(W,Z))(\h{\pi}_{110}(W,Z)-\pi_{110}(W,Z)) = o_p(n^{-1/2})$
        \item $\mathbb{P}_1(\h{\mu}_{\cdot s0}(W)-\mu_{\cdot s0}(W))^2 = o_p(n^{-1/2})$
        \item $\mathbb{P}_1(\h{\mu}_{\cdot10}(W)-\mu_{\cdot10}(W))^2 = o_p(n^{-1/2})$
        \item (Positivity) $p_0(z_s,w)>0$ and $p_0(w,z)>0$ almost everywhere, such that the density ratios $\pi_{1s0}(w,z_s)$ and $\pi_{110}(w,z)$ are well-defined and between $(0,1)$.
    \end{itemize}
\end{condition}
\begin{theorem}
    For variable subset $s$, suppose $v_{Z}^{\denom}(s) >  0$ and Condition~\ref{cond:condcov} hold.
    Then the estimator $\hat{v}_{Z}(s)$ is asymptotically normal.
    \label{thrm:cond_cov_detailed}
\end{theorem}

\section{Proofs}
\label{sec:proofs}
\textbf{Notation.} For all proofs, we will write $\mathbb P$ to mean $\mathbb P^{\Tt}$ (and likewise for the empirical version) for notational simplicity.

\textbf{Overview of derivation strategy.} We first present the general strategy for proving asymptotic normality of the estimators for the decompositions.
Details on nonparametric debiased estimation can be found in texts such as \citet{Tsiatis2006-ay} and \citet{Kennedy2022-to}.

Let $v(\mathbb{P})$ be a pathwise differentiable quantity that is a function of the true regular (differentiable in quadratic mean) probability distribution $\mathbb{P}$ over random variable $O$.
For instance, $v$ in the case of mean is defined as $v(\mathbb{P}):=\E_{o\sim \mathbb{P}(O)}[o]$.
Let $\h{\mathbb{P}}$ denote an arbitrary regular estimator of $\mathbb{P}$, such as the maximum likelihood estimator.
The plug-in estimator is then defined as $v(\h{\mathbb{P}})$. 

The von-Mises expansion of the functional $v$ (which linearizes $v$ in analogy to the first-order Taylor expansion), given it is pathwise differentiable, gives
\begin{equation}
    \label{eq:linearization}
    v(\h{\mathbb{P}}) - v(\mathbb{P}) = -\mathbb{P}\ \psi(o;\h{\mathbb{P}}) + R(\h{\mathbb{P}},\mathbb{P}).
\end{equation} 
Here, the function $\psi$ is called an influence function (or a functional gradient of $v$ at $\h{\mathbb{P}}$). $R(\h{\mathbb{P}},\mathbb{P})$ is a second-order remainder term. The one-step corrected estimators we consider have the form of $v(\h{\mathbb{P}}) + \mathbb{P}_n \psi(o;\h{\mathbb{P}})$ where $\mathbb{P}_n$ denotes a sample average. Following the expansion above, the one-step corrected estimator can be analyzed as follows,
\begin{align*}
    &\left(v(\h{\mathbb{P}}) + \mathbb{P}_n \psi(o;\h{\mathbb{P}})\right) - v(\mathbb{P})\\
    &= (\mathbb{P}_n - \mathbb{P})\psi(o;\h{\mathbb{P}}) + R(\h{\mathbb{P}},\mathbb{P})\\
    &= (\mathbb{P}_n - \mathbb{P})\psi(o;\mathbb{P}) + (\mathbb{P}_n - \mathbb{P})(\psi(o;\h{\mathbb{P}}) - \psi(o;\mathbb{P})) + R(\h{\mathbb{P}},\mathbb{P})
\end{align*}
Our goal will be to analyze each of the three terms and to show that they are asymptotically negligible at $\sqrt{n}$-rate, such that the one-step corrected estimator satisfies
$$\left(v(\h{\mathbb{P}}) + \mathbb{P}_n \psi(o;\h{\mathbb{P}})\right)- v(\mathbb{P}) =  \mathbb{P}_n \psi(o;\mathbb{P}) + o_p(n^{-1/2}),$$
where we used the property of influence functions that they have zero mean.
Thus the one-step corrected estimator is asymptotically normal with mean $v(\mathbb{P})$ and variance $var(\psi(o;\mathbb{P}))/n$, which allows for the construction of CIs.
In the following proofs, we present the influence functions without derivations; see \citet{Kennedy2022-to} and \citet{Hines2022-tr} for strategies for deriving influence functions.

\subsection{Aggregate decompositions}
\label{sec:result_agg}
Let the nuisance parameters in the one-step estimators $\Lambda_{\W},\Lambda_{\Z},\Lambda_{\Y}$ be denoted by $\eta_{\W}=(\mu_{\cdot 00},\pi_{100}), \eta_{\Z}=(\mu_{\cdot \cdot 0},\mu_{\cdot 00},\pi_{110}), \eta_{\Y}=(\mu_{\cdot \cdot 0},\pi_{110})$ respectively. Denote the estimated nuisances by $\h \eta_{\W}, \h \eta_{\Z}, \h \eta_{\Y}$. The canonical gradients for the three estimands are
\begin{align}
    \psi_{\W}(d,w,z,y;\eta_{\W}) = &
    \left[\left(\ell(w,z,y) - {\mu}_{\cdot 00}(w)\right) \pi_{100}(w) - \ell(w,z,y)\right]\fracdzero + {\mu}_{\cdot 00}(w)\fracdone - \Lambda_\W \label{eq:influence_agg_baseline}\\
    \psi_{\Z}(d,w,z,y;\eta_{\Z}) = & 
    \left[\left(\ell(w,z,y) - \mu_{\cdot \cdot 0}(w,z)\right) \pi_{110}(w,z)\right]\fracdzero + \mu_{\cdot \cdot 0}(w,z)\fracdone \nonumber\\
    & - \left[\left(\ell(w,z,y) - \mu_{\cdot 00}(w)\right) \pi_{100}(w)\right]\fracdzero + \mu_{\cdot 00}(w)\fracdone - \Lambda_\Z \label{eq:influence_agg_cov}
    \\
    \psi_{\Y}(d,w,z,y;\eta_{\Y}) = & \left(\ell(w,z,y) - \mu_{\cdot \cdot 0}(w,z)\right)\fracdone - \left[\left(\ell(w,z,y) - \mu_{\cdot \cdot 0}(w,z)\right) \pi_{110}(w,z)\right]\fracdzero - \Lambda_\Y. \label{eq:influence_agg_outcome}
    \end{align}

\begin{theorem}[Theorem~\ref{thm:aggregate}]
    Under conditions outlined in Theorem~\ref{thm:aggregate}, the one-step corrected estimators for the aggregate decomposition terms, baseline, conditional covariate, and conditional outcome $\h{\Lambda}_\W$, $\h{\Lambda}_\Z$, and $\h{\Lambda}_\Y$, are asymptotically linear, i.e.
    \begin{align}
        \h{\Lambda}_{\N} -\Lambda_{\N} = \mathbb P_n \psi_{\N} + o_p(n^{-1/2}) \quad \forall \N \in \{\W,\Z,\Y\}.
    \end{align}
\end{theorem}
\begin{proof}
    The estimands $\Lambda_\W,\Lambda_\Z,\Lambda_\Y$ have similarities to the standard average treatment effect (ATE) in the causal inference literature (see \citep[][Example 2]{Kennedy2022-to}. Hence, the estimators and their asymptotic properties directly follow. 
    For treatment $T$, outcome $O$, and confounders $C$, the mean outcome under $T=1$ among the population with $T=0$ is identified as
    \begin{align}
    \phi = \int o p(o|c,t=1)p(c|t=0) do dc
    \label{eq:ate_def}
    \end{align}
and its one-step corrected estimator can be derived from the canonical gradient of $\phi$, which takes the following form after plugging in the estimates of the nuisance models:
    $$
    \hat \phi_n
    = \mathbb{P}_{n} \left\{\frac{\mathbbm{1}\{T=1\}}{\Pr(T=1)}\hat{\pi}(c) \left( O - \hat{\mu}_1(c) \right) + \frac{\mathbbm{1}\{T=0\}}{\Pr(T=0)} \hat{\mu}_1(c) \right\}
    $$
    satisfies
    $$
    \hat\phi_n - \phi = 
    \mathbb{P}_n \left\{
    \frac{\mathbbm{1}\{T=1\}}{\Pr(T=1)} {\pi}(c) \left( O - {\mu}_1(c) \right) + \frac{\mathbbm{1}\{T=0\}}{\Pr(T=0)} {\mu}_1(c)
    -\phi \right\} + o_p(n^{-1/2})
    $$
    where $\mu_1(c)=\E[o|c,t=1]$ and $\pi(c)=p(c|t=0)/p(c|t=1)$
    as long as the following conditions hold:
    \begin{itemize}
        \item $p(c|t=1)>0$ almost everywhere such that the density ratios $\pi(c)$ are well-defined and bounded,
        \item $\mathbb{P}_1 (\h{\mu}_1 - \mu_1)(\h{\pi} - \pi) = o_p(n^{-1/2})$.
    \end{itemize}

    We establish the estimators and their influence functions by showing that they can all be viewed as mean outcomes of the form \eqref{eq:ate_def}.

    \textbf{Baseline term $\Lambda_\W$}.
    The first term $\mathbb{E}_{100}\left[\ell(w,z,y) \right]$ is a mean outcome with respect to $p(\ell(w,z,y)|w,d=0)p(w|d=1)$, which is the same as that in \eqref{eq:ate_def} but with $\ell(w,z,y)$ as the outcome, $w$ as the confounder, and $d$ as the (flipped) treatment.
    The second term $\mathbb{E}_{000}\left[\ell(w,z,y) \right]$ is a simple average over $D=0$ population whose influence function is the $\ell(w,z,y)$ itself.
    
\textbf{Conditional covariate term $\Lambda_\Z$}.
    First term $\mathbb{E}_{110}\left[\ell(w,z,y)\right]$ is the mean outcome with respect to $p(\ell(w,z,y)|w,z,d=0)p(w,z|d=1)$, where the chief difference is $(w,z)$ is the confounder.
    Second term $\mathbb{E}_{1 0 0}\left[\ell(w,z,y)\right]$ is also a mean outcome, as discussed above.

    \textbf{Conditional outcome term $\Lambda_\Y$}. 
    First term $\mathbb{E}_{111}\left[\ell(w,z,y)\right]$ is a simple average over the $D=1$ population.
\end{proof}

\subsection{Value of $s$-partial conditional covariate shifts}
\label{sec:result_detailed_cov}
Let nuisance parameters in the one-step estimator $v_{\Z,s}^{\num}$ be denoted $\eta_{\Z,s}^{\num} = (\mu_{\cdot s 0}, \mu_{\cdot 10},\mu_{\cdot 0_{-s}0},\mu_{001},\mu_{\cdot \cdot 0}, \pi_{1s0}, \pi_{110})$ and the set of estimated nuisances by $\h \eta_{\Z,s}^{\num}$.
The canonical gradient of $v_{\Z}^{\num}(s)$ is 
\begin{align}
    \psi_{\Z,s}^{\num}(D,W,Z,Y; \eta_{\Z,s}^{\num}) &= (\mu_{\cdot s 0}(W) - \mu_{\cdot 1 0}(W))^2\frac{\mathbbm{1}\{D=1\}}{\Pr(D=1)} \nonumber\\
    & + 2(\mu_{\cdot s 0}(W) - \mu_{\cdot 1 0}(W)) (\ell - {\mu}_{\cdot0_{-s}0}(W,Z_s)){\pi}_{1s0}(W,Z_s) \frac{\mathbbm{1}\{D=0\}}{\Pr(D=0)} \nonumber\\
    & - 2(\mu_{\cdot s 0}(W) - \mu_{\cdot 1 0}(W)) (\ell - {\mu}_{\cdot\cdot 0}(W,Z)){\pi}_{110}(W,Z)\frac{\mathbbm{1}\{D=0\}}{\Pr(D=0)} \nonumber\\    
    &+ 2(\mu_{\cdot s 0}(W) - \mu_{\cdot 1 0}(W)) ({\mu}_{\cdot0_{-s}0}(W,Z_s) - {\mu}_{\cdot s0}(W))
    \frac{\mathbbm{1}\{D=1\}}{\Pr(D=1)} \nonumber\\
    &- 2 (\mu_{\cdot s 0}(W) - \mu_{\cdot 1 0}(W)) ({\mu}_{\cdot \cdot 0}(W,Z) - {\mu}_{\cdot 10}(W))
    \frac{\mathbbm{1}\{D=1\}}{\Pr(D=1)} \nonumber\\
    & - v_{\Z}^{\num}(s)\label{eq:influence_num_cov}.
\end{align}

\begin{lemma}
    Under Condition~\ref{cond:condcov}, $\h{v}_{\Z}^{\num}(s)$ satisfies
    \begin{align}
        \h{v}_{\Z}^{\num}(s) - v_{\Z}^{\num}(s) &= \mathbb{P}_n \psi_{\Z,s}^{\num}(D,W,Z,Y; \eta_{\Z,s}^{\num}) + o_p(n^{-1/2})\\
        \h v^{\denom}_\Z - v^{\denom}_\Z &= \mathbb{P}_n \psi_{\Z,\emptyset}^{\num}(D,W,Z,Y;\eta_{\Z,s}^{\num}) + o_p(n^{-1/2})
\end{align}
    \label{lemma:numerator_cond_cov}
\end{lemma}
\begin{proof}
    Consider the following decomposition
    \begin{align}
        &\h{v}_{\Z}^{\num}(s) - v_{\Z}^{\num}(s) \nonumber \\
        = &(\mathbb{P}_n - \mathbb{P})\psi_{\Z}^{\num}(D,W,Z,Y;\eta_{\Z,s}^{\num})
        \label{eq:b1_cond_cov_num}
        \\
        & + (\mathbb{P}_n - \mathbb{P})(\psi_{\Z}^{\num}(D,W,Z,Y;\h\eta_{\Z,s}^{\num})-\psi_{\Z}^{\num}(W,Z,Y;\eta_{\Z,s}^{\num}))
        \label{eq:b2_cond_cov_num}
        \\
        & + \mathbb{P}(\psi_{\Z}^{\num}(D,W,Z,Y;\h\eta_{\Z,s}^{\num}) - \psi_{\Z}^{\num}(D,W,Z,Y;\eta_{\Z,s}^{\num}))
        \label{eq:b3_cond_cov_num}
    \end{align}
We note that \eqref{eq:b2_cond_cov_num} converges to a normal distribution per CLT assuming the variance of $\psi_{\Z,s}^{\num}$ is finite. The empirical process term \eqref{eq:b2_cond_cov_num} is asymptotically negligible, as the nuisance parameters $\eta_{\Z,s}^{\num}$ are estimated using a separate training data split from the evaluation data and \citep[][Lemma 1]{Kennedy2022-to} states that
$$(\mathbb P_n - \mathbb P)(\psi_{\Z,s}^{\num}(D,W,Z,Y;\h \eta_{\Z,s}^{\num})-\psi_{\Z,s}^{\num}(D,W,Z,Y;\eta_{\Z,s}^{\num}) = o_p(n^{-1/2})$$
as long as estimators for all nuisance parameters are consistent.
We now establish that the remainder term \eqref{eq:b3_cond_cov_num} is also asymptotically negligible. Integrating with respect to $Y$, we have that
\begin{align}
    \eqref{eq:b3_cond_cov_num} =& \mathbb P \ 2(\h{\mu}_{\cdot s0}(W)-\h{\mu}_{\cdot10}(W))
    \begin{aligned}[t]
        \times \Big{(}& (\mu_{\cdot0_{-s}0}(Z_s,W) - \h{\mu}_{\cdot0_{-s}0}(Z_s,W))\h{\pi}_{1s0}(W,Z_s)\frac{\mathbbm{1}\{D=0\}}{p(D=0)}\\
        &+ (\h{\mu}_{\cdot 0_{-s}0}(Z_s,W) - \h{\mu}_{\cdot s0}(W))\frac{\mathbbm{1}\{D=1\}}{p(D=1)}\\
        &- (\mu_{\cdot\cdot0}(W,Z) - \h{\mu}_{\cdot\cdot0}(W,Z))\h{\pi}_{110}(W,Z)\frac{\mathbbm{1}\{D=0\}}{p(D=0)}\\        
        &- (\h{\mu}_{\cdot\cdot0}(W,Z) - \h{\mu}_{\cdot10}(W))\frac{\mathbbm{1}\{D=1\}}{p(D=1)}\Big{)}\\
    \end{aligned}\\
    &+ \mathbb P ((\h{\mu}_{\cdot s0}(W)-\h{\mu}_{\cdot10}(W))^2 - (\mu_{\cdot s0}(W)-\mu_{\cdot10}(W))^2) \frac{\mathbbm{1}\{D=1\}}{p(D=1)}\\
    = & \mathbb P \ 2(\h{\mu}_{\cdot s0}(W)-\h{\mu}_{\cdot10}(W))
    \begin{aligned}[t]
        \times \Big{(}& (\mu_{\cdot0_{-s}0}(Z_s,W) - \h{\mu}_{\cdot0_{-s}0}(Z_s,W))\left(\h{\pi}_{1s0}(W,Z_s)-{\pi}_{1s0}(W,Z_s)\right)\frac{\mathbbm{1}\{D=0\}}{p(D=0)}\\
        & + (\mu_{\cdot0_{-s}0}(Z_s,W) - \h{\mu}_{\cdot0_{-s}0}(Z_s,W)){\pi}_{1s0}(W,Z_s)\frac{\mathbbm{1}\{D=0\}}{p(D=0)}\\
        &+ (\h{\mu}_{\cdot 0_{-s}0}(Z_s,W) - \h{\mu}_{\cdot s0}(W))\frac{\mathbbm{1}\{D=1\}}{p(D=1)}\\
        &- (\mu_{\cdot\cdot0}(W,Z) - \h{\mu}_{\cdot\cdot0}(W,Z))\left(\h{\pi}_{110}(W,Z)-{\pi}_{110}(W,Z) \right)\frac{\mathbbm{1}\{D=0\}}{p(D=0)}\\
        &- (\mu_{\cdot\cdot0}(W,Z) - \h{\mu}_{\cdot\cdot0}(W,Z)){\pi}_{110}(W,Z)\frac{\mathbbm{1}\{D=0\}}{p(D=0)}\\    
        &- (\h{\mu}_{\cdot\cdot0}(W,Z) - \h{\mu}_{\cdot10}(W))\frac{\mathbbm{1}\{D=1\}}{p(D=1)}\Big{)}\\
    \end{aligned}\\
    &+ \mathbb P ((\h{\mu}_{\cdot s0}(W)-\h{\mu}_{\cdot10}(W))^2 - (\mu_{\cdot s0}(W)-\mu_{\cdot10}(W))^2) \frac{\mathbbm{1}\{D=1\}}{p(D=1)}
\end{align}
From convergence conditions in Condition~\ref{cond:condcov}, this simplifies to
\begin{align}
    \eqref{eq:b3_cond_cov_num} = & \mathbb P \ 2(\h{\mu}_{\cdot s0}(W)-\h{\mu}_{\cdot10}(W))
    \begin{aligned}[t]
        \times \Big{(}& (\mu_{\cdot0_{-s}0}(Z_s,W) - \h{\mu}_{\cdot0_{-s}0}(Z_s,W)){\pi}_{1s0}(W,Z_s)\frac{\mathbbm{1}\{D=0\}}{p(D=0)}\\
        &+ (\h{\mu}_{\cdot 0_{-s}0}(Z_s,W) - \h{\mu}_{\cdot s0}(W))\frac{\mathbbm{1}\{D=1\}}{p(D=1)}\\
        &- (\mu_{\cdot\cdot0}(W,Z) - \h{\mu}_{\cdot\cdot0}(W,Z)){\pi}_{110}(W,Z)\frac{\mathbbm{1}\{D=0\}}{p(D=0)}\\    
        &- (\h{\mu}_{\cdot\cdot0}(W,Z) - \h{\mu}_{\cdot10}(W))\frac{\mathbbm{1}\{D=1\}}{p(D=1)}\Big{)}\\
    \end{aligned}\\
    &+ \mathbb P ((\h{\mu}_{\cdot s0}(W)-\h{\mu}_{\cdot10}(W))^2 - (\mu_{\cdot s0}(W)-\mu_{\cdot10}(W))^2) \frac{\mathbbm{1}\{D=1\}}{p(D=1)}\\
    & + o_p(n^{-1/2}),
\end{align}
Given the true density ratios, we can further simplify the expectations over $D=0$ weighted by the density ratios in the expression above to expectations over $D=1$. By  definition of $\mu_{\cdot0_{-s}0}(Z_s,W)$ in \eqref{eq:mu_0ms0_cond_cov} and ${\mu}_{\cdot s0}(W)$ in  \eqref{eq:mu_dots0_cond_cov} and the definition of $\mu_{\cdot\cdot0}(W,Z)$ and $\mu_{\cdot 10}(W)$ in Section~\ref{sec:agg_est_inf},  \eqref{eq:b3_cond_cov_num} simplifies to
\begin{align*}
\eqref{eq:b3_cond_cov_num} = & \mathbb P_1 \ 2(\h{\mu}_{\cdot s0}(W)-\h{\mu}_{\cdot10}(W))
        (\mu_{\cdot s0}(W) - \h{\mu}_{\cdot s0}(W))\\
    & - \mathbb P_1 \ 2(\h{\mu}_{\cdot s0}(W)-\h{\mu}_{\cdot10}(W))(\mu_{\cdot 10}(W) - \h{\mu}_{\cdot10}(W))\\
    &+ \mathbb P_1 (\h{\mu}_{\cdot s0}(W)-\h{\mu}_{\cdot10}(W))^2 - ({\mu}_{\cdot s0}(W)-{\mu}_{\cdot10}(W))^2)\\
    & + o_p(n^{-1/2}),
\end{align*}
which is $o_p(n^{-1/2})$ as long as the convergence conditions in Condition~\ref{cond:condcov} hold.

As the denominator $v^{\denom}_\Z$ is equal to the numerator $v^{\num}_\Z(\emptyset)$, it follows that the one-step estimator for the denominator $\h v^{\denom}_{\Z}$ is asymptotically linear with influence function $\psi_{\Z,\emptyset}^{\num}$.
\end{proof}

\begin{proof}[Proof for Theorem~\ref{thrm:cond_cov_detailed}]
    Combining Lemma~\ref{lemma:numerator_cond_cov} and the Delta method \citep[][Theorem 3.1]{Van_der_Vaart1998-cj}, the estimator $\h v_{\Z}(s)=\h v_{\Z}^{\num}(s)/\h v_{\Z}^{\denom}$ is asymptotically linear
    \begin{align*}
        \frac{\hat v_{\Z}^{\num}(s)}{\hat{v}_{\Z}^{\denom}}
    - \frac{v_{\Z}^{\num}(s)}{v_{\Z}^{\denom}}
    = \mathbbm{P}_n \psi_{\Z,s}(D,W,Z,Y;\eta_{\Z,s}^\num, \eta_{\Z}^\denom) + o_p(n^{-1/2}),
    \end{align*}
    with influence function
    \begin{align}
    \label{eq:influence_detailed_cov}
    \psi_{\Z,s}(D,W,Z,Y;\eta_{\Z,s}^\num, \eta_{\Z,s}^\denom) =\frac{1}{v_{\Z}^{\denom}} \psi_{\Z,s}^{\num}(D,W,Z,Y;\eta_{\Z,s}^\num) - \frac{v_{\Z}^{\num}(s)}{(v_{\Z}^{\denom})^2} \psi_{\Z}^{\denom}(D,W,Z,Y;\eta_{\Z}^{\denom}),
    \end{align}
    where $\psi_{\Z,s}^{\num}(D,W,Z,Y; \eta_{\Z,s}^{\num})$ is defined in \eqref{eq:influence_num_cov} and
    $\psi_{\Z}^{\denom}(D,W,Z,Y;\eta_{\Z}^{\denom}) = \psi_{\Z,\emptyset}^{\num}(D,W,Z,Y; \eta_{\Z,\emptyset}^{\num})$.
    
    Accordingly, the estimator asymptotically follows the normal distribution,
    \begin{align}
        \sqrt{n}\left(
        \h v_{\Z}(s) - v_{\Z}(s)
        \right) \rightarrow_d
        N(0, \text{var}(\psi_{\Z,s}(D,W,Z,Y;\eta_{\Z,s}^\num,\eta_{\Z,s}^{\denom}))
    \end{align}
\end{proof}

\subsection{Value of $s$-partial conditional outcome shifts}
\label{sec:result_detailed_outcome}
Let the nuisance parameters in $v_{\Y, \binned}^{\num}$ be denoted
${\eta}_{\Y,s}^{\num}=(Q_{\binned},\mu_{\cdot\cdot1},\mu_{\cdot\cdot s}, \pi)$ and its estimate as $\hat{\eta}_{\Y,s}^{\num}$.

We represent the one-step corrected estimator for $v_{\Y,\binned}^{\num}(s)$ as the V-statistic
\begin{align}
\hat{v}_{\Y,\binned}^{\num}(s)= & \mathbb{P}_{1,n} \tilde{\mathbb{P}}_{1,n} \left(\hat{\mu}_{\cdot\cdot 1}(W,Z)-\hat{\mu}_{\cdot\cdot s}(W,Z)\right)^2\\
    & + 2 \mathbb{P}_{1,n} \tilde{\mathbb{P}}_{1,n} \left(\hat{\mu}_{\cdot\cdot 1}(W,Z)-\hat{\mu}_{\cdot\cdot s}(W,Z)\right) (\ell - \mu_{\cdot \cdot 1}(W,Z))\\
    & - 2\mathbb{P}_{1,n} \tilde{\mathbb{P}}_{1,n}\left(\hat{\mu}_{\cdot\cdot 1}(W,Z_s,\tilde{Z}_{-s})-\hat{\mu}_{\cdot\cdot s}(W,Z_s,\tilde{Z}_{-s})\right)
    (\ell( W, Z_s,\tilde{Z}_{-s}, Y) - \mu_{\cdot \cdot s}(W,Z_s,\tilde{Z}_{-s}))\pi(\tilde{Z}_{-s},Z_s,W, q_{\binned}(W,Z))\\
= & \mathbb{P}_{1,n} \tilde{\mathbb{P}}_{1,n}h\left(W,Z,Y,\tilde{W},\tilde{Z},\tilde{Y};\hat{\eta}_{\Y,s}^{\num}\right).
\label{eq:outcome_h}
\end{align}

In more detail, the conditions in Theorem~\ref{theorem:conditional_outcome_detailed} are as follows.
\begin{condition}
For variable subset $s$, suppose the following hold
\begin{itemize}
    \item $\pi(W,Z_s,Z_{-s},Q_{\binned})$ is bounded
    \item $\hat{\pi}$ is consistent
    \item $\mathbb{P}_1 \left(\hat \mu_{\cdot \cdot 0} - \mu_{\cdot \cdot 0}\right)^2 = o_p(n^{-1/2})$
    \item $\mathbb{P}_1 \left(\hat \mu_{\cdot \cdot 1} - \mu_{\cdot \cdot 1}\right)^2 = o_p(n^{-1/2})$
    \item $\mathbb{P}_1 \left(\hat \mu_{\cdot \cdot s} - \mu_{\cdot \cdot s}\right)^2 = o_p(n^{-1/2})$
    \item $\mathbb{P}_1 \left(\hat q_\binned - q_\binned\right)^2 = o_p(n^{-1})$
    \item $\mathbb{P}_1 \left(\hat \mu_{\cdot \cdot s} - \mu_{\cdot \cdot s}\right)\left(\hat \pi - \pi \right) = o_p(n^{-1/2})$.
\end{itemize}
\label{cond:conditional_outcome_detailed}
\end{condition}

\begin{lemma}
    Assuming Condition~\ref{cond:conditional_outcome_detailed} holds, $\h v_{\Y,\binned}^\num$ is an asymptotically linear estimator for $v_{\Y,\binned}^\num$, i.e. 
    \begin{align}
        \h{v}_{\Y,\binned}^{\num}(s) - v_{\Y,\binned}^{\num}(s) = \mathbb P_{1,n} \psi_{\Y,s}^{\num}(D,W,Z,Y; \eta_{\Y,s}^{\num}) + o_p(n^{-1/2}),
    \end{align}
    with influence function
    \begin{align}
        \psi_{\Y,s}^{\num}\left(d,w,z,y;\eta_{\Y,s}^{\num}\right) = & (\mu_{\cdot\cdot 1}(w,z) - \mu_{\cdot\cdot s}(w,z))^{2}\nonumber\\
         & +2(\mu_{\cdot\cdot 1}(w,z) - \mu_{\cdot\cdot s}(w,z))
         \left[\ell(w,z,y)-\mu_{\cdot\cdot 1}(w,z)\right]\nonumber\\
         & -2\mathbb{P}_1\left(\mu_{\cdot\cdot 1}(w,z_{s},Z_{-s})-\mu_{\cdot \cdot s}(w,z_{s},Z_{-s})\right)\left[\ell\left(w,z_s,Z_{-s},y)\right)-\mu_{\cdot \cdot s}(w,z_{s},Z_{-s})\right]
         \pi\left({Z}_{-s}, z_{s},w,q_{\binned}(w,z) \right)\nonumber\\
         & - v_{\Y,\binned}^{\num}(s)\label{eq:influence_cond_outcome_num}.
    \end{align}
    \label{lemma:numerator_cond_outcome}
\end{lemma}
\begin{proof}
Defining the symmetrized version of $h$ in \eqref{eq:outcome_h} as $h_{sym}(W,Z,Y,\tilde W,\tilde Z,\tilde Y)=\frac{h(W,Z,Y,\tilde W,\tilde Z,\tilde Y)+h(\tilde W,\tilde Z,\tilde Y,W,Z,Y)}{2}$,
we rewrite the estimator as 
\[
\hat{v}_{\Y,\binned}^{\num}(s)
=\mathbb{P}_{1,n} \tilde{\mathbb{P}}_{1,n}
h_{sym}\left(W,Z,Y,\tilde{W},\tilde{Z},\tilde{Y};\hat{\eta}_{\Y,s}^{\num}\right).
\]

Per Theorem 12.3 in \citep{Van_der_Vaart1998-cj}, the H\'ajek projection of $\hat{v}_{\Y,\binned}^{\num}(s)$ is
\begin{align*}
\hat{u}_{\Y,\binned}^{num}(s) & =\sum_{i=1}^{n}\mathbb{P}_{1}\left[\mathbb{P}_{1,n}\tilde{\mathbb{P}}_{1,n}h_{sym}\left(W,Z,Y,\tilde{W},\tilde{Z},\tilde{Y};\hat \eta_{\Y,s}^{\num}\right)-\bar{\hat{v}}_{Y}^{num}(s)\mid X_{i},Y_{i}\right]\\
 & =\sum_{i=1}^{n}\mathbb{P}_{1}\left[h_{sym}\left(X_{i},Y_{i},X^{(2)},Y^{(2)};\hat \eta_{\Y,s}^{\num}\right)-\bar{\hat{v}}_{Y}^{num}(s)\mid X_{i},Y_{i}\right]\\
 & =\sum_{i=1}^{n}{h}_{sym,1}\left(X_{i},Y_{i};\hat \eta_{\Y,s}^{\num} \right)
\end{align*}
where $\bar{\hat{v}}_{Y}^{num}(s) = \mathbb{P}_1\tilde{\mathbb{P}}_1 h_{sym}\left(W,Z,Y,\tilde{W},\tilde{Z},\tilde{Y};\hat \eta_{\Y,s}^{\num}\right)$.

Consider the decomposition

\begin{align}
\hat{v}_{\Y,\binned}^{\num}(s)-v_{\Y,\binned}^{\num}(s)= & \mathbb{P}_{1,n}\tilde{\mathbb{P}}_{1,n}h_{sym}\left(W,Z,Y,\tilde{W},\tilde{Z},\tilde{Y};\hat{\eta}_{\Y,s}^{\num}\right)-\mathbb{P}_{1} \tilde{\mathbb{P}}_{1}h_{sym}\left(W,Z,Y,\tilde{W},\tilde{Z},\tilde{Y};\eta_{\Y,s}^{\num}\right)\nonumber \\
= & \mathbb{P}_{1,n}\tilde{\mathbb{P}}_{1,n}{h}_{sym}\left(W,Z,Y,\tilde{W},\tilde{Z},\tilde{Y};\hat{\eta}_{\Y,s}^{\num}\right)-\mathbb{P}_{1,n}\left[{h}_{sym,1}\left(X,Y;\hat{\eta}_{\Y,s}^{\num}\right)+\bar{\hat{v}}_{\Y}^{\num}(s)\right]
\label{eq:b1_outcome}\\
 & +\left(\mathbb{P}_{1,n}-\mathbb{P}_{1}\right)\left({h}_{sym,1}\left(X,Y;\h \eta_{\Y,s}^{\num}\right)+\bar{\hat{v}}_{\Y}^{\num}(s)-h_{sym,1}\left(X,Y\right)-v_{\Y}^{\num}(s)\right)
 \label{eq:b2_outcome}\\
 & +\left(\mathbb{P}_{1,n}-\mathbb{P}_{1}\right)\left(h_{sym,1}\left(X,Y\right)+v_{\Y}^{\num}(s)\right)
 \label{eq:b3_outcome}\\
 & +\mathbb{P}_{1}\left({h}_{sym,1}\left(X,Y;\h \eta_{\Y,s}^{\num}\right)+\bar{\hat{v}}_{\Y}^{\num}(s)-h_{sym,1}\left(X,Y;\eta_{\Y,s}^{\num}\right)-v_{\Y}^{\num}(s)\right).
 \label{eq:b4_outcome}
\end{align}

We analyze each term in turn.

\uline{Term \eqref{eq:b1_outcome}}: Suppose $\mathbb{P}_{1}{h}_{sym}^{2}(W,Z,Y,\tilde{W},\tilde{Z},\tilde{Y};\hat \eta_{\Y,s}^{\num})<\infty$.
Via a straightforward extension of the proof in Theorem 12.3 in \citet{Van_der_Vaart1998-cj}, one can show that
\[
\frac{var\left(\mathbb{P}_{1,n}\tilde{\mathbb{P}}_{1,n}{h}_{sym}\left(W,Z,Y,\tilde{W},\tilde{Z},\tilde{Y}; \hat \eta_{\Y,s}^{\num}\right)\right)}{var\left(\mathbb{P}_{1,n}{h}_{sym,1}\left(W,Z,Y; \hat \eta_{\Y,s}^{\num}\right)\right)}\rightarrow_{p}1.
\]

Then by Theorem 11.2 in \citet{Van_der_Vaart1998-cj} and Slutsky's lemma, we have
\[
\mathbb{P}_{1,n}\tilde{\mathbb{P}}_{1,n}{h}_{sym}\left(W,Z,Y,\tilde{W},\tilde{Z},\tilde{Y}; \hat \eta_{\Y,s}^{\num}\right)-\mathbb{P}_{1,n}\left[{h}_{sym,1}\left(W,Z,Y; \hat \eta_{\Y,s}^{\num}\right)+\bar{\hat{v}}_{\Y}^{\num}(s)\right]=o_{p}\left(n^{-1/2}\right).
\]

\uline{Term \eqref{eq:b2_outcome}:} We perform sample splitting to estimate the nuisance
parameters and calculate the estimator for $\hat{v}_{\Y}^{\num}(s)$.
Then by Lemma 1 in \citet{Kennedy2022-to}, we have that
\[
\left(\mathbb{P}_{1,n}-\mathbb{P}_{1}\right)\left({h}_{sym,1}\left(W,Z,Y; \hat \eta_{\Y,s}^{\num}\right)+\bar{\hat{v}}_{\Y}^{\num}(s)-h_{sym,1}\left(W,Z,Y; \eta_{\Y,s}^{\num}\right)-v_{\Y}^{\num}(s)\right)=o_{p}(n^{-1/2})
\]
as long as the estimators for the nuisance parameters are consistent.

\uline{Term \eqref{eq:b3_outcome}:} This term $\left(\mathbb{P}_{1,n}-\mathbb{P}_{1}\right)\left(h_{sym,1}\left(W,Z,Y; \eta_{\Y,s}^{\num}\right)+v_{\Y}^{\num}(s)\right)=\left(\mathbb{P}_{1,n}-\mathbb{P}_{1}\right)h_{sym,1}\left(W,Z,Y; \eta_{\Y,s}^{\num}\right)$
follows an asymptotic normal distribution per CLT.

\uline{Term \eqref{eq:b4_outcome}:} We will show that this bias term is asymptotically negligible.
For notational simplicity, let $\hat \xi(W,Z_s,Z_{-s}) = \hat{\mu}_{\cdot\cdot 1}(W,Z)-\hat{\mu}_{\cdot\cdot s}(W,Z)$.
\begin{align}
 & \mathbb{P}_{1} \tilde{\mathbb{P}}_{1}\left({h}_{sym}\left(W,Z,Y,\tilde{W},\tilde{Z},\tilde{Y}; \hat \eta_{\Y,s}^{\num}\right)+\bar{\hat{v}}_{\Y}^{\num}(s)-h_{sym}\left(W,Z,Y,\tilde{W},\tilde{Z},\tilde{Y}; \eta_{\Y,s}^{\num}\right)-v_{\Y}^{\num}(s)\right)\nonumber \\
= & \mathbb{P}_{1} \tilde{\mathbb{P}}_{1}\left({h}\left(W,Z,Y,\tilde{W},\tilde{Z},\tilde{Y}; \hat \eta_{\Y,s}^{\num}\right)-h\left(W,Z,Y,\tilde{W},\tilde{Z},\tilde{Y}; \eta_{\Y,s}^{\num}\right)\right)\nonumber \\
= & \mathbb{P}_{1}\left(\hat{\mu}_{\cdot\cdot 1}(W,Z)-\hat{\mu}_{\cdot\cdot s}(W,Z)\right)^{2}-\mathbb{P}_{1}\left(\mu_{\cdot\cdot 1}(W,Z)-\mu_{\cdot\cdot s}(W,Z)\right)^{2}\nonumber \\
 & +2\mathbb{P}_{1}\left(\hat{\mu}_{\cdot\cdot 1}(W,Z)-\hat{\mu}_{\cdot\cdot s}(W,Z)\right)\left[\mu_{\cdot\cdot1}(W,Z)-\hat{\mu}_{\cdot\cdot 1}(W,Z)\right]\nonumber \\
 & -2\mathbb{P}_{1} \tilde{\mathbb{P}}_{1}\left(\hat{\mu}_{\cdot\cdot 1}(W,Z_{s},\tilde{Z}_{-s})-\hat{\mu}_{\cdot\cdot s}(W,Z_{s},\tilde{Z}_{-s})\right)\left[\ell(W,Z_{s},\tilde{Z}_{-s},Y)-\hat{\mu}_{\cdot\cdot s}(W,Z_{s},\tilde{Z}_{-s})\right]\hat{\pi}\left(\tilde{Z}_{-s},Z_{s},W,\hat{q}_{\binned}(W,Z)\right)\nonumber \\
= & \mathbb{P}_{1}\left(\hat{\mu}_{\cdot\cdot 1}(W,Z)-\hat{\mu}_{\cdot\cdot s}(W,Z)\right)^{2}-\mathbb{P}_{1}\left(\mu_{\cdot\cdot 1}(W,Z)-\mu_{\cdot\cdot s}(W,Z)\right)^{2}\nonumber \\
 & +2\mathbb{P}_{1}\left(\hat{\mu}_{\cdot\cdot 1}(W,Z)-\hat{\mu}_{\cdot\cdot s}(W,Z)\right)\left[\mu_{\cdot\cdot1}(W,Z)-\hat{\mu}_{\cdot\cdot 1}(W,Z)\right]\nonumber \\
 & -2\mathbb{P}_{1} \tilde{\mathbb{P}}_{1}
 \hat \xi(W,Z_{s},\tilde{Z}_{-s})
 \left[\mu_{\cdot\cdot s}(W,Z_{s},\tilde{Z}_{-s})-\hat{\mu}_{\cdot\cdot s}(W,Z_{s},\tilde{Z}_{-s})\right]\hat{\pi}\left(\tilde{Z}_{-s},Z_{s},W,{q}_{\binned}(W,Z)\right)\nonumber \\
& -2\mathbb{P}_{1} \tilde{\mathbb{P}}_{1}\hat \xi(W,Z_{s},\tilde{Z}_{-s})
\left[\ell(W,Z_{s},\tilde{Z}_{-s},Y)-\hat{\mu}_{\cdot\cdot s}(W,Z_{s},\tilde{Z}_{-s})\right] 
\left[\hat{\pi}\left(\tilde{Z}_{-s},Z_{s},W,\hat{q}_{\binned}(W,Z)\right)
- \hat{\pi}\left(\tilde{Z}_{-s},Z_{s},W,{q}_{\binned}(W,Z)\right)\right]
\nonumber \\
= & \mathbb{P}_{1}\left(\mu_{\cdot\cdot s}(W,Z)-\hat{\mu}_{\cdot\cdot s}(W,Z)\right)\left(\hat{\mu}_{\cdot\cdot 1}(W,Z)-\hat{\mu}_{\cdot\cdot s}(W,Z)+\mu_{\cdot\cdot 1}(W,Z)-\mu_{\cdot\cdot s}(W,Z)\right) 
\label{eq:cond_out_l1}\\
 & +\mathbb{P}_{1}\left(\hat{\mu}_{\cdot\cdot 1}(W,Z)-\mu_{\cdot\cdot 1}(W,Z)\right)\left(\mu_{\cdot\cdot 1}(W,Z)-\mu_{\cdot\cdot s}(W,Z)-\hat{\mu}_{\cdot\cdot 1}(W,Z)+\hat{\mu}_{\cdot\cdot s}(W,Z)\right)
 \label{eq:cond_out_l2}\\
 & -2\mathbb{P}_{1} \tilde{\mathbb{P}}_{1}
 \hat \xi(W,Z_{s},\tilde{Z}_{-s})
 \left[\mu_{\cdot\cdot s}(W,Z_{s},\tilde{Z}_{-s})-\hat{\mu}_{\cdot\cdot s}(W,Z_{s},\tilde{Z}_{-s})\right]\hat{\pi}\left(\tilde{Z}_{-s},Z_{s},W,{q}_{\binned}(W,Z)\right)
 \label{eq:cond_out_l3}\\
 & -2\mathbb{P}_{1} \tilde{\mathbb{P}}_{1}\hat \xi(W,Z_{s},\tilde{Z}_{-s})
\left[\ell(W,Z_{s},\tilde{Z}_{-s},Y)-\hat{\mu}_{\cdot\cdot s}(W,Z_{s},\tilde{Z}_{-s})\right] 
\left[\hat{\pi}\left(\tilde{Z}_{-s},Z_{s},W,\hat{q}_{\binned}(W,Z)\right)
- \hat{\pi}\left(\tilde{Z}_{-s},Z_{s},W,{q}_{\binned}(W,Z)\right)\right]
\label{eq:cond_out_l4}.
\end{align}
Note that \eqref{eq:cond_out_l4} is $o_p(n^{-1/2})$ under the assumed convergence rates for $\hat{q}_{\binned}$.
In addition, \eqref{eq:cond_out_l2} is $o_{p}(n^{-1/2})$, under the assumed convergence rates for $\hat\mu_{\cdot\cdot 1}$ and $\hat\mu_{\cdot\cdot s}$.

Analyzing the remaining summands \eqref{eq:cond_out_l1} + \eqref{eq:cond_out_l3}, we note that it simplifies as follows:
\begin{align*}
& \mathbb{P}_{1}\left(\mu_{\cdot\cdot s}(X)-\hat{\mu}_{\cdot\cdot s}(X)\right)\left(\hat{\mu}_{\cdot\cdot 1}(X)-\hat{\mu}_{\cdot\cdot s}(X)+\mu_{\cdot\cdot 1}(X)-\mu_{\cdot\cdot s}(X)\right)\\
& -2\mathbb{P}_{1} \tilde{\mathbb{P}}_{1}\hat\xi(W,Z_s,\tilde{Z}_{-s})\left[\mu_{\cdot\cdot s}(W,Z_s,\tilde{Z}_{-s})-\hat{\mu}_{\cdot\cdot s}(W,Z_s,\tilde{Z}_{-s})\right]\left(\hat{\pi}\left(\tilde{Z}_{-s},Z_s,W,q_{\binned}(W,Z))\right)-\pi\left(\tilde{Z}_{-s},Z_s,W,q_{\binned}(W,Z))\right)\right)\\
 & -2\mathbb{P}_{1} \tilde{\mathbb{P}}_{1}\hat\xi(W,Z_s,\tilde{Z}_{-s})\left[\mu_{\cdot\cdot s}(W,Z_s,\tilde{Z}_{-s})-\hat{\mu}_{\cdot\cdot s}(W,Z_s,\tilde{Z}_{-s})\right]\pi\left(\tilde{Z}_{-s},Z_s,W,q_{\binned}(W,Z))\right)\\
= & \mathbb{P}_{1}\left(\mu_{\cdot\cdot s}(X)-\hat{\mu}_{\cdot\cdot s}(X)\right)\left(\mu_{\cdot\cdot 1}(X)-\hat{\mu}_{\cdot\cdot 1}(X)-\mu_{\cdot\cdot s}(X)+\hat{\mu}_{\cdot\cdot s}(X)\right)\\
& -2\mathbb{P}_{1} \tilde{\mathbb{P}}_{1}\hat\xi(W,Z_s,\tilde{Z}_{-s})\left[\mu_{\cdot\cdot s}(W,Z_s,\tilde{Z}_{-s})-\hat{\mu}_{\cdot\cdot s}(W,Z_s,\tilde{Z}_{-s})\right]\left(\hat{\pi}\left(\tilde{Z}_{-s},Z_s,W,q_{\binned}(W,Z))\right)-\pi\left(\tilde{Z}_{-s},Z_s,W,q_{\binned}(W,Z))\right)\right),
\end{align*}
which is $o_{p}(n^{-1/2})$, under the assumed convergence rates for $\hat{\mu}_{\cdot\cdot s}$, $\hat{\mu}_{\cdot\cdot 1}$, and $\hat{\pi}$.
\end{proof}

\begin{condition}[Convergence conditions for $\h v_{\Y}^{\denom}$]
\label{cond:denominator_cond_outcome}
    Suppose the following holds
    \begin{itemize}
        \item $\mathbb{P}_1 (\mu_{\cdot \cdot 1} - \mu_{\cdot \cdot 0} - (\hat{\mu}_{\cdot \cdot 1} - \h \mu_{\cdot \cdot 0}))^2 = o_p(n^{-1/2})$
        \item $\mathbb{P}_0 (\mu_{\cdot \cdot 0} - \hat{\mu}_{\cdot \cdot 0})(\pi_{110} - \h \pi_{110}) = o_p(n^{-1/2})$
        \item $\mathbb{P}_0 (\mu_{\cdot \cdot 1} - \mu_{\cdot \cdot 0})^2$ is bounded
        \item (Positivity) $p(w,z|d=0)>0$ almost everywhere, such that the density ratios $\pi_{110}(w,z)$ are well-defined and bounded.
    \end{itemize}
\end{condition}

Let the nuisance parameters in the one-step estimator $v_{\Y}^{\denom}$ be denoted by $\eta_{\Y}^{\denom}=(\mu_{\cdot\cdot 0},\mu_{\cdot\cdot 1},\pi_{110})$ and the set of estimated nuisances by $\h \eta_{\Y}^{\denom}$.

\begin{lemma}
    Assuming Condition~\ref{cond:denominator_cond_outcome} holds, then $\h v_{\Y}^{\denom}$ is an asymptotically linear estimator for $v_{\Y}^{\denom}$, i.e.
    \begin{align*}
        \h v_{\Y}^{\denom} - v_{\Y}^{\denom} = \mathbb{P}_n \psi_{\Y}^{\denom}(D,W,Z,Y;\eta_{\Y}^{\denom}) + o_p(n^{-1/2})
    \end{align*}
    with influence function
    \begin{align}
        \psi_{\Y}^{\denom}(D,W,Z,Y; \eta_{\Y}^{\denom}) = & 
        \left(\mu_{\cdot\cdot 1}(W,Z) - \mu_{\cdot\cdot 0}(W,Z)\right)^2 \frac{\mathbbm{1}\{D=1\}}{p(D=1)} \\
        &\ + 2 \left(\mu_{\cdot\cdot 1}(W,Z) - \mu_{\cdot\cdot 0}(W,Z)\right) (\ell - \mu_{\cdot\cdot 1}(W,Z)) \frac{\mathbbm{1}\{D=1\}}{p(D=1)} \\
        &\ - 2 \left(\mu_{\cdot\cdot 1}(W,Z) - \mu_{\cdot\cdot 0}(W,Z)\right) (\ell - \mu_{\cdot\cdot 0}(W,Z)) \pi_{110}(W,Z) \frac{\mathbbm{1}\{D=0\}}{p(D=0)} \\
        &\ - v_{\Y}^{\denom}\label{eq:influence_cond_outcome_denom}.
    \end{align}
    \label{lemma:denom_cond_outcome}
\end{lemma}
\begin{proof}
Consider the following decomposition of bias in the one-step corrected estimate
    \begin{align}
        &\h v_{\Y}^{\denom} - v_{\Y}^{\denom}\nonumber\\
        =& (\mathbb{P}_n - \mathbb{P})\psi_{\Y}^{\denom}(D,W,Z,Y;\eta_{\Y}^{\denom})\label{eq:b1_cond_outcome_denom}\\ 
        &+ (\mathbb{P}_n - \mathbb{P})(\psi_{\Y}^{\denom}(D,W,Z,Y;\h \eta_{\Y}^{\denom}) - \psi_{\Y}^{\denom}(D,W,Z,Y;\eta_{\Y}^{\denom}))\label{eq:b2_cond_outcome_denom}\\
        &+ \mathbb{P} (\psi_{\Y}^{\denom}(D,W,Z,Y;\h \eta_{\Y}^{\denom}) - \psi_{\Y}^{\denom}(D,W,Z,Y;\eta_{\Y}^{\denom}))\label{eq:b3_cond_outcome_denom}
    \end{align}
We observe that \eqref{eq:b1_cond_outcome_denom} converges to a normal distribution per CLT assuming that the variance of $\psi_{\Y}^{\denom}$ is finite. The empirical process term \eqref{eq:b2_cond_outcome_denom} is asymptotically negligible since the nuisance parameters $\eta_{\Y}^{\denom}$ are evaluated on an separate evaluation data split from the training data used for estimation. In addition assuming that the estimators for the nuisance parameters are consistent, \citet[][Lemma 1]{Kennedy2022-to} states that
$$(\mathbb{P}_n - \mathbb{P})(\psi_{\Y}^{\denom}(D,W,Z,Y;\h \eta_{\Y}^{\denom}) - \psi_{\Y}^{\denom}(D,W,Z,Y;\eta_{\Y}^{\denom})) = o_p(n^{-1/2}).$$

We now show that the remainder term \eqref{eq:b3_cond_outcome_denom} is also asymptotically negligible. Substituting the influence function and integrating with respect to $Y$, \eqref{eq:b3_cond_outcome_denom} becomes
\begin{align}
    \eqref{eq:b3_cond_outcome_denom} = & \mathbb{P}\left(\hat\mu_{\cdot \cdot 1}(W,Z) - \hat\mu_{\cdot \cdot 0}(W,Z) - (\mu_{\cdot \cdot 1}(W,Z) - \mu_{\cdot \cdot 0}(W,Z))\right)^2\frac{\mathbbm{1}(D=1)}{p(D=1)}\\
    & + 2 \mathbb{P}\left(\hat\mu_{\cdot \cdot 1}(W,Z) - \hat\mu_{\cdot \cdot 0}(W,Z) - (\mu_{\cdot \cdot 1}(W,Z) - \mu_{\cdot \cdot 0}(W,Z))\right)(\mu_{\cdot \cdot 1}(W,Z) - \mu_{\cdot \cdot 0}(W,Z))\frac{\mathbbm{1}(D=1)}{p(D=1)}\\
    &\ + 2 \mathbb{P} \left(\hat\mu_{\cdot \cdot 1}(W,Z) - \hat\mu_{\cdot \cdot 0}(W,Z)\right) (\mu_{\cdot \cdot 1}(W,Z) - \hat\mu_{\cdot \cdot 1}(W,Z)) \frac{\mathbbm{1}\{D=1\}}{p(D=1)} \\
    &\ - 2 \mathbb{P} \left(\hat\mu_{\cdot \cdot 1}(W,Z) - \hat\mu_{\cdot \cdot 0}(W,Z)\right) (\mu_{\cdot \cdot 0}(W,Z) - \hat\mu_{\cdot \cdot 0}(W,Z)) \hat\pi_{110}(W,Z) \frac{\mathbbm{1}\{D=0\}}{p(D=0)}\\
    = & \mathbb{P}\left(\hat\mu_{\cdot \cdot 1}(W,Z) - \hat\mu_{\cdot \cdot 0}(W,Z) - (\mu_{\cdot \cdot 1}(W,Z) - \mu_{\cdot \cdot 0}(W,Z))\right)^2\frac{\mathbbm{1}(D=1)}{p(D=1)}\\
    & + 2 \mathbb{P}(\mu_{\cdot \cdot 0}(W,Z) - \hat\mu_{\cdot \cdot 0}(W,Z))(\mu_{\cdot \cdot 1}(W,Z) - \mu_{\cdot \cdot 0}(W,Z))\frac{\mathbbm{1}(D=1)}{p(D=1)}\\
    &\ - 2 \mathbb{P} \left(\hat\mu_{\cdot \cdot 1}(W,Z) - \hat\mu_{\cdot \cdot 0}(W,Z)\right) (\mu_{\cdot \cdot 0}(W,Z) - \hat\mu_{\cdot \cdot 0}(W,Z)) (\hat\pi_{110}(W,Z) - \pi_{110}(W,Z)) \frac{\mathbbm{1}\{D=0\}}{p(D=1)}\\
    &\ - 2 \mathbb{P} \left(\hat\mu_{\cdot \cdot 1}(W,Z) - \hat\mu_{\cdot \cdot 0}(W,Z)\right) (\mu_{\cdot \cdot 0}(W,Z) - \hat\mu_{\cdot \cdot 0}(W,Z)) \pi_{110}(W,Z) \frac{\mathbbm{1}\{D=0\}}{p(D=0)}\\
    = & \mathbb{P}\left(\hat\mu_{\cdot \cdot 1}(W,Z) - \hat\mu_{\cdot \cdot 0}(W,Z) - (\mu_{\cdot \cdot 1}(W,Z) - \mu_{\cdot \cdot 0}(W,Z))\right)^2\frac{\mathbbm{1}(D=1)}{p(D=1)}\\
    &\ - 2 \mathbb{P} \left(\hat\mu_{\cdot \cdot 1}(W,Z) - \hat\mu_{\cdot \cdot 0}(W,Z)\right) (\mu_{\cdot \cdot 0}(W,Z) - \hat\mu_{\cdot \cdot 0}(W,Z)) (\hat\pi_{110}(W,Z) - \pi_{110}(W,Z)) \frac{\mathbbm{1}\{D=0\}}{p(D=0)}
\end{align}
Thus the remainder term is $o_p(n^{-1/2})$ if Condition~\ref{cond:denominator_cond_outcome} holds.
\end{proof}

\begin{proof}[Proof for Theorem~\ref{theorem:conditional_outcome_detailed}]
    Combining Lemmas~\ref{lemma:numerator_cond_outcome}, \ref{lemma:denom_cond_outcome}, and the Delta method \citep[][Theorem 3.1]{Van_der_Vaart1998-cj}, the estimator $\h v_{\Y,\binned}(s)=\h v_{\Y,\binned}^{\num}(s)/\h v_{\Y}^{\denom}$ is asymptotically linear
    \begin{align*}
        \frac{\hat v_{\Y,\binned}^{\num}(s)}{\hat{v}_{\Y}^{\denom}}
    - \frac{v_{\Y,\binned}^{\num}(s)}{v_{\Y}^{\denom}}
    = \mathbbm{P}_n \psi_{\Y,\binned,s}(D,W,Z,Y;\eta_{\Y,s}^\num,\eta_{\Y}^{\denom}) + o_p(n^{-1/2}),
    \end{align*}
    with influence function
    \begin{align}
    \label{eq:influence_cond_outcome}
    \psi_{\Y,\binned,s}(D,W,Z,Y;\eta_{\Y,s}^\num,\eta_{\Y}^{\denom}) =\frac{1}{v_{\Y}^{\denom}} \psi_{\Y,s}^{\num}(D,W,Z,Y;\eta_{\Y,s}^\num) - \frac{v_{\Y,\binned}^{\num}(s)}{(v_{\Y}^{\denom})^2} \psi_{\Y}^{\denom}(D,W,Z,Y;\eta_{\Y}^{\denom}),
    \end{align}
    where $\psi_{\Y,s}^{\num}$ and $\psi_{\Y}^{\denom}$ are defined in \eqref{eq:influence_cond_outcome_num} and \eqref{eq:influence_cond_outcome_denom}.

    Accordingly, the estimator follows a normal distribution asymptotically,
    \begin{align}
        \sqrt{n} \left(
        \h v_{\Y}(s) - v_{\Y}(s)
        \right) \rightarrow_d
        N(0, \text{var}(\psi_{\Y,\binned,s}(D,W,Z,Y;\eta_{\Y}^{\num},\eta_{\Y}^{\denom}))
    \end{align}
\end{proof}

\section{Implementation details}
\label{sec:implement}
Here we describe how the nuisance parameters can be estimated in each of the decompositions.
In general, density ratio models can be estimated via a standard reduction to a classification problem where a probabilistic classifier is trained to discriminate between source and target domains \citep{Sugiyama2007-pk}.

\textbf{Note on computation time}. Shapley value computation can be parallelized over the subsets. For high-dimensional tabular data, grouping together variables can further reduce computation time (and increase interpretability).

\subsection{Aggregate decompositions}
\label{sec:implement_agg}
\textbf{Density ratio models.} Using direct importance estimation \citep{Sugiyama2007-pk}, density ratio models $\pi_{100}(W)$ and $\pi_{110}(W,Z)$ can be estimated by fitting classifiers on the combined source and target data to predict $D=0$ or $1$ from features $W$ and $(W,Z)$, respectively.

\textbf{Outcome models.} The outcome models $\mu_{\cdot00}(W)$ and $\mu_{\cdot\cdot0}(W,Z)$ can be fit in a number of ways. One option is to estimate the conditional distribution of the outcome (i.e. $p_0(Y|W)$ or $p_0(Y|W,Z)$) using binary classifiers, from which one can obtain an estimate of the conditional expectation of the loss.
Alternatively, one can estimate the conditional expectations of the loss directly by fitting regression models.

\subsection{Detailed decomposition for $s$-partial outcome shift}
\label{sec:implement_detailed_outcome}

\textbf{Density ratio models.} The density ratio $\pi(W,Z_s,Z_{-s},Q)=p_{1}(Z_{-s}|W,Z_s,q(W,Z)=Q)/p_{1}(Z_{-s})$ in \eqref{eq:value_num_outcome} can be estimated as follows.
Create a second (``phantom'') dataset of the target domain in which $Z_{-s}$ is independent of $Z_s$ by permuting the original $Z_{-s}$ in the target domain.
Compute $q_\binned$ for all observations in the original dataset and the permuted dataset.
Concatenate the original dataset from the target domain with the permuted dataset.
Train a classifier to predict if an observation is from the original versus the permuted dataset.

\textbf{Outcome models.} The outcome models $\mu_{\cdot \cdot 1}$ and $\mu_{\cdot \cdot s}(W,Z)$ can be similarly fit by estimating the conditional distribution $p_0(Y|W,Z)$ and $p_s(Y|W,Z,q_{\binned}(W,Z))$ on the target domain, and then taking expectation of the loss.

\textbf{Computing U-statistics.} Calculating the double average $\mathbb{P}_{1,n}^{\Tt}\tilde{\mathbb{P}}_{1,n}^{\Tt}$ in the estimator requires evaluating all $n^2$ pairs of data points in target domain.
This can be computationally expensive, so a good approximation is to subsample the inner average. We take 2000 subsamples. We did not see large changes in the bias of the estimates compared to calculating the exact U-statistics.

\subsection{Detailed decomposition for $s$-partial conditional outcome shift}
\label{sec:implement_detailed_cov}
\textbf{Density ratio models.} The ratio $\pi_{1s0}(W,Z_s)=p_1(W,Z_s)/p_0(W,Z_s)$ can be similarly fit using direct importance estimation.

\textbf{Outcome models.} We require the following models.
\begin{itemize}
    \item $\mu_{\cdot 0_{-s} 0}(z_s,w)=\mathbb{E}_{\cdot 00}[\ell|z_{s},w]$, defined in \eqref{eq:mu_0ms0_cond_cov}, can be estimated by regressing loss against $w,z_s$ on the source domain.
    \item $\mu_{\cdot 1 0}(w)=\mathbb{E}_{\cdot 10}[\ell|w]$, defined in \eqref{eq:mu_dot10_cond_cov}, can be estimated by regressing $\mu_{\cdot \cdot 0}(w,z)$ against $w$ in the target domain.
    \item $\mu_{\cdot s 0}(w)=\mathbb{E}_{\cdot s0}[\ell|z_{s},w]$, defined in \eqref{eq:mu_dots0_cond_cov}, can be estimated by regressing $\mu_{\cdot 0_{-s} 0}(z_s,w)$ against $w$ in the target domain.
\end{itemize}

For all models, we use cross-validation to select among model types and hyperparameters. Model selection is important so that the convergence rate conditions for the asymptotic normality results are met.

\section{Simulation details}
\label{sec:simulation_details}
\textbf{Data generation}: We generate synthetic data under two settings. For the coverage checks in Section~\ref{sec:coverage}, all features are sampled independently from a multivariate normal distribution. The mean of the $(W,Z)$ in the source and target domains are $(0,2,0.7,3)$ and $(0,0,0,0)$, respectively.
The outcome in the source and target domains are simulated from a logistic regression model with coefficients $(0.3,1,0.5,1)$ and $(0.3,0.1,0.5,1.4)$. 

In the second setting for baseline comparisons in Figure~\ref{fig:comparators}b, each feature in $W\in\mathbb{R}$ and $Z\in\mathbb{R}^5$ is sampled independently from the rest from a uniform distribution over $[-1,1)$. The binary outcome $Y$ is sampled from a logistic regression model with coefficients $(0.2,0.4,2,0.25,0.1,0.1)$ in source and $(0.2,-0.4,0.8,0.1,0.1,0.1)$ in target.

\textbf{Sample-splitting}: We fit all models on 80\% of the data points from both source and target datasets which is the $\Tr$ partition, and keep the remaining 20\% for computing the estimators which is the $\Tt$ partition.

\textbf{Model types}: We use \texttt{scikit-learn} implementations for all models \citep{pedregosa2011scikit}. We use 3-fold cross validation to select models. For density models, we fit random forest classifiers and logistic regression models with polynomial features of degree 3. Depending on whether the target outcome in outcome models is binary or real-valued, we fit random forest classifiers or regressors, and logistic regression or linear regression models with ridge penalty. Specific hyperparameter ranges for the grid search will be provided in the code, to be made publicly available.

\textbf{Computing time and resources}: Computation for the VI estimates can be quite fast, as Shapley value computation can be parallelized over the subsets and
the number of unique variable subsets sampled in the Shapley value approximation is often quite small. For instance, for the ACS Public Coverage case study with 34 features, the unique subsets is $131$ even when the number of sampled subsets is $3000$, and it takes around 160 seconds to estimate the value of a single variable subset. All experiments are run on a 2.60 GHz processor with 8 CPU cores.

\begin{figure}
    \centering
    
    \includegraphics[width=0.5\textwidth]{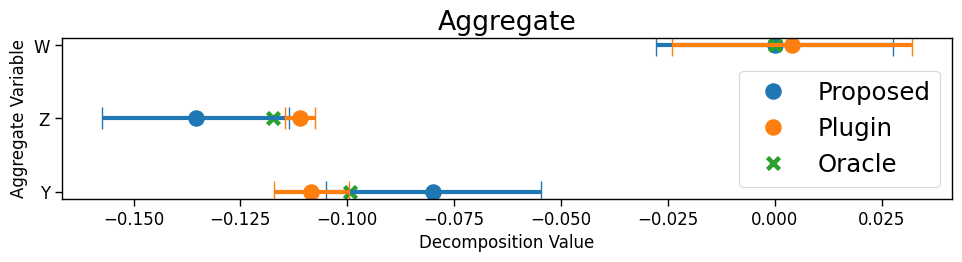}
    \includegraphics[width=0.5\textwidth]{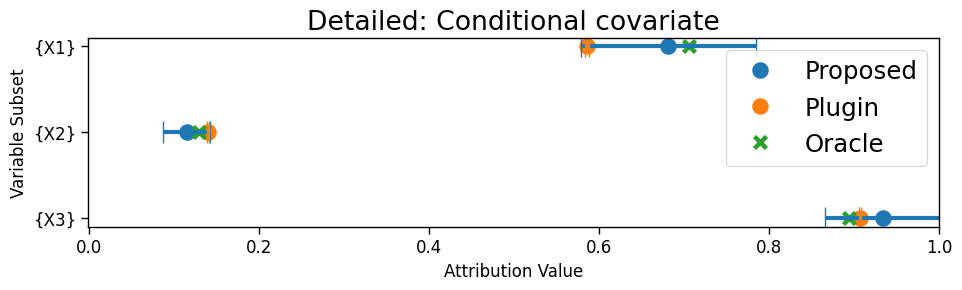}
    
    \vspace{-0.3cm}
    \includegraphics[width=0.5\textwidth]{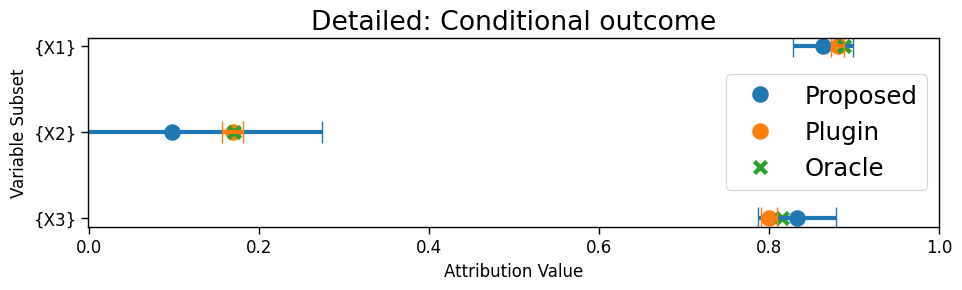}
    \caption{Sample estimates and CIs for simulation from Section~\ref{sec:coverage}. \texttt{Proposed} is debiased ML estimator for HDPD.}
    \label{fig:simulation_example}
\end{figure}

\section{Data analysis details}
\label{sec:data_details}

\textbf{Synthetic.} We describe accuracy of the ML algorithm after it is retrained with the top $k$ features and predictions from the original model.

\textbf{Hospital readmission.} Using data from the electronic health records of a large safety-net hospital in the US, we analyzed the transferability of performance measures of a Gradient Boosted Tree (GBT) trained to predict 30-day readmission risk for the general patient population (source) but applied to patients diagnosed with heart failure (target).
Each of the source and target datasets have 3750 observations for analyzing the performance gap.
The GBT is trained on a held-out sample of 18,873 points from the general population.
Features include  4 demographic variables ($W$) and 16 diagnosis codes ($Z$).
While training, we reweigh samples by class weights to address class imbalance.

\textbf{ACS Public Coverage.} We extract data from the American Community Survey (ACS) to predict whether a person has public health insurance. 
The data only contains persons of age less than 65 and having an income of less than \$30,000.
We analyze a neural network (MLP) trained on data from Nebraska (source) to data from Louisiana (target) given 3000 and 6000 observations from the source and target domains, respectively.
Another 3300 from source for training the model. Figure~\ref{fig:casestudy-acs-full} shows the detailed decomposition of conditional outcome shift for the dataset.

\begin{figure}
    \centering
    \includegraphics[scale=0.4]{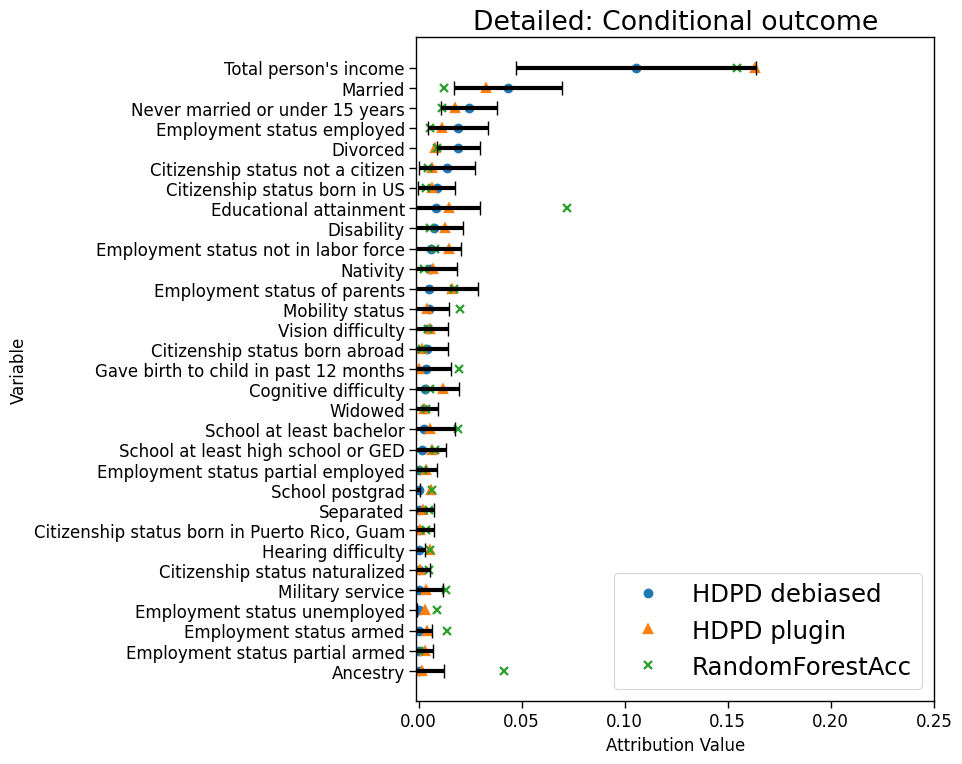}
    \caption{
    Detailed decompositions for the performance gap of a model predicting insurance coverage prediction across two US states (NE $\rightarrow$ LA). Plot shows values for the full set of 31 covariates.
    }
    \label{fig:casestudy-acs-full}
\end{figure}

\begin{table}
    \centering
    \begin{tabular}{rrrr}
\toprule
k & Diff AUC-k & Lower CI & Upper CI \\
\midrule
1 & 0.000 & 0.000 & 0.000 \\
2 & \textbf{0.006} & 0.001 & 0.010 \\
3 & -0.002 & -0.007 & 0.002 \\
4 & 0.004 & -0.002 & 0.008 \\
5 & -0.001 & -0.006 & 0.003 \\
6 & -0.002 & -0.008 & 0.003 \\
7 & \textbf{0.007} & 0.002 & 0.011 \\
8 & \textbf{0.006} & 0.001 & 0.010 \\
\bottomrule
\end{tabular}
    \caption{Difference in AUCs between the revised insurance prediction model with respect to the top $k$ variables identified by the proposed versus RandomForestAcc procedures (Diff AUC-k = Proposed $-$ RandomForestAcc). 95\% CIs are shown.}
    \label{tab:acs_auc}
\end{table}

\end{document}